\documentclass{article}

\usepackage{microtype}
\usepackage{graphicx}
\usepackage{subcaption}
\usepackage{booktabs} %

\usepackage{hyperref}

\usepackage[preprint]{icml2026}

\usepackage{amsmath}
\usepackage{amssymb}
\usepackage{mathtools}
\usepackage{amsthm}

\usepackage{fancyvrb}

\usepackage[capitalize,noabbrev]{cleveref}

\usepackage{comment}
\usepackage{graphicx}                %
\usepackage{booktabs}                %
\usepackage{microtype}  
\usepackage{hyperref}
\usepackage{enumitem}
\usepackage{multirow}

\theoremstyle{plain}
\newtheorem{theorem}{Theorem}[section]
\newtheorem{proposition}[theorem]{Proposition}

\theoremstyle{definition}
\newtheorem{definition}[theorem]{Definition}

\theoremstyle{remark}

\usepackage[table]{xcolor}
\definecolor{academicBlue}{rgb}{0.28,0.39,0.55}

\usepackage[textsize=tiny]{todonotes}

\icmltitlerunning{VSB}

\begin{document}

\twocolumn[
  \icmltitle{
When to Commit? Towards Variable-Size Self-Contained Blocks for Discrete Diffusion Language Models
}

  \begin{icmlauthorlist}
     \icmlauthor{Danny Wang}{sch}
    \icmlauthor{Ruihong Qiu}{sch}
     \icmlauthor{Zi Huang}{sch}
    \\
    \vspace{0.4em}
    \icmlauthor{\mdseries $^\text{1}$The University of Queensland}{}
  \end{icmlauthorlist}

  \icmlkeywords{Machine Learning, ICML}

  \vskip 0.3in
]

\printAffiliationsAndNotice{}  %

\begin{abstract}
Discrete diffusion language models (dLLMs) enable parallel token updates with bidirectional attention, yet practical generation typically adopts blockwise semi-autoregressive decoding. This switch creates a training-inference mismatch: \textbf{training} denoises with \textbf{full-sequence} context, while \textbf{inference} commits tokens within a \textbf{bounded block without future context}. Therefore, decoding with fixed-size or heuristic-based blocks can lead to premature token commitments, as decisions are made without full access to future context that could alter those choices. Motivated by this, we propose \textbf{self-containedness} as a principled criterion for block commitment. A block is self-contained if its predictions remain consistent with Future-Aware (FA) or without No-Future (NF) access to future context, reframing block boundary selection as a test of self-containedness rather than a heuristic choice.
Based on this principle, we introduce Variable-size Self-contained Blocks (VSB) for dLLMs. VSB scores and selects block boundaries using the \textbf{divergence between token-level predictive distributions under NF and FA} conditioning, which quantifies how predictions would change if future context were revealed. 
We provide theoretical justification linking self-containedness to predictive consistency, and extensive experiments validate VSB's efficacy over fixed-size and heuristic blockwise decoding. 
\end{abstract}

\section{Introduction}
\label{sec:intro}

Recent progress in discrete diffusion language models (dLLM) enable parallel token generation with global bidirectional attention, breaking the serial dependency of autoregression~\cite{nie2025llada,dream,diffusion_survey,liu2025wedlm,ni2025dllm_vs_ar}. 
In practice, many dLLMs are decoded \textbf{blockwise} in a semi-autoregressive manner: the model denoises a block of $B$ tokens in parallel and then often irreversibly commits the block to the prefix, meaning the committed tokens cannot be changed later when future context becomes available, before moving on to generate the next block.~\cite{Arriola2025BlockDiff,Chen2025dParallel,fastdllm}. This strategy balances parallel efficiency with autoregressive style coherence, however, it raises a \textbf{central but under-explored question:}

\textit{When can a block be safely committed such that its tokens remain consistent when future context is revealed?}

Existing blockwise decoding typically decides when to commit simply by choosing block boundary with either fixed block size or surface heuristics (e.g., punctuation). These rules can commit tokens whose predictions would change once later tokens are generated, yielding \textit{future-dependent} commitments and unstable conditioning for subsequent blocks.
~\cite{lu2025adablock,Li2025DAEDAL,dllmvar}.

\begin{figure}[t!]
    \centering
    \includegraphics[width=0.8\linewidth]{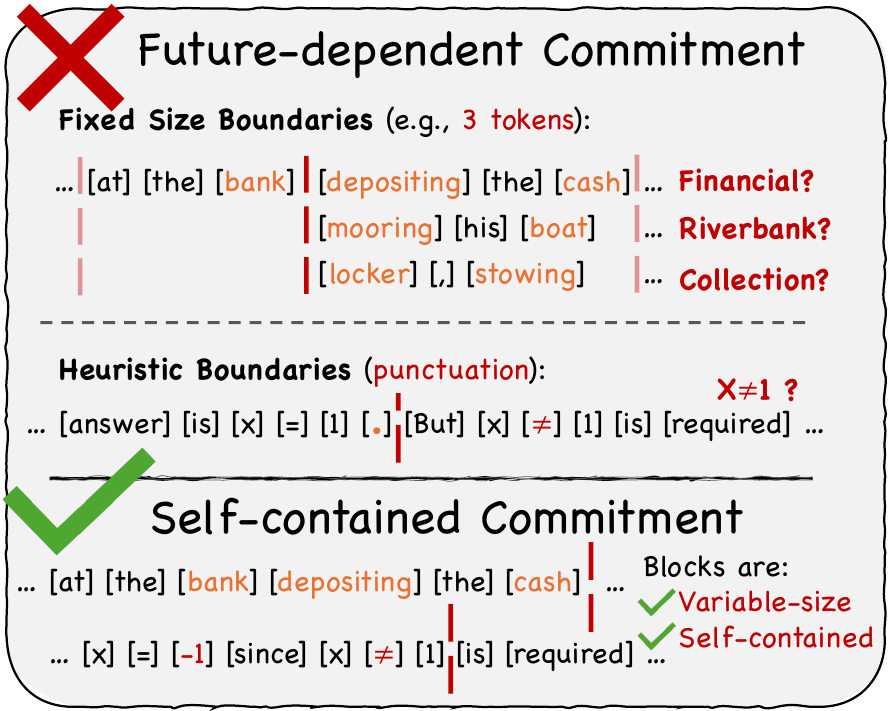}
    \caption{\textbf{Future-dependent vs. self-contained commitments in dLLMs.} Fixed or heuristic boundaries can commit tokens whose meaning rely on future context, while self-contained commitment preserves semantic coherence and reduces future dependence.}
    \label{fig:committing_comparison}
    \vspace{-0.6cm}
\end{figure}

This typical blockwise schema exposes two coupled issues:
\textbf{(1) Fundamental training-inference mismatch.} Standard dLLMs are typically trained to denoise under \textit{full-sequence} ($x_{1:\text{full}}$) bidirectional context: when predicting a token, the model can condition on both its left prefix and right suffix~\cite{soft_masked_dllm,nie2025llada,llada1_5,MDLM}. This training setup is natural for full-length bidirectional decoding, as it encourages the model to use information from the entire sequence for making predictions. During typical blockwise semi-AR decoding, by contrast, committing a block is often \textit{irreversible}: once tokens are committed, they cannot be revised using later context~\cite{fastdllm,Arriola2025BlockDiff,lu2025adablock}. If the model would substantially change its predictions for tokens in the current block after seeing future tokens, the block is not self-contained and should not be committed.
\textbf{(2) Heuristic block boundaries.} Most methods either fix the block size $B$ (e.g., 32 tokens) or align boundaries with surface cues such as punctuation~\cite{lu2025adablock,Xiong2025Stepwiser,Arriola2025BlockDiff}. Neither choice ensures that the committed block is semantically complete or independent of future context. As illustrated in Figure~\ref{fig:committing_comparison}, a fixed-size boundary can cut through an important semantic unit, and  punctuation cue (e.g., a full stop) does not imply that the preceding phrase is resolved without the following clause. Both cases lead to \textit{future-dependent} commitments that can propagate errors to subsequent blocks.
Recent work has started to reduce this mismatch with blockwise-aligned training objectives, but boundary selection typically remains heuristic, without an explicit principled notion of when a block is ready to be committed~\citep{sun2025blockwisesft,Chen2025dParallel,fastdllm_v2}.

In light of these challenges, we propose \textbf{self-containedness} as the criterion for placing block boundaries in semi-AR decoding. 
A candidate block is self-contained if the model's predictive distributions for tokens inside the block remain consistent when future context is revealed. This directly matches the irreversibility of commitment: we should commit tokens only when their predictions \textbf{no longer rely on future information beyond the boundary.}

However, self-containedness is inherently counterfactual because future tokens are unavailable at inference time. We resolve this by comparing two predictive distributions under the \textit{same prefix} within a fixed \textit{ block budget} of length $W$: (i) \textbf{Future-Aware (FA)} predictions, where the model conditions on the entire budgeted block, and (ii) \textbf{No-Future (NF)} predictions, where the model conditions only up to  candidate boundary $b$.
\textbf{If the predictive distributions for the tokens to be committed are consistent between FA and NF, then committing up to $b$ is self-contained 
}; if the distributions diverge, the boundary is premature. 
This motivates quantifying self-containedness by measuring the divergence between conditioning with and without future information (NF-FA), enabling \textbf{self-contained variable-size blocks} without fixed lengths or punctuation heuristics.

We introduce \textbf{Variable-Size Self-Contained Blocks (VSB)} to realise this principle. VSB (i) aligns training with blockwise inference by training under the same \textit{No-Future} truncation used at commitment time, and (ii) dynamically selects self-contained block boundaries by scoring candidate cutoffs via NF-FA divergence. Our contributions are:
\begin{itemize}[leftmargin=*,labelsep=0.0em,itemsep=0pt,parsep=0pt,topsep=0pt]
    \item \textbf{Principle:} We formalise \textbf{\textit{Self-Containedness}} as a criterion for determining block boundaries in blockwise semi-AR decoding. A block is committed only when revealing future tokens does not largely alter its predictive distributions.
    \item \textbf{Method:} We propose \textbf{VSB}, a \textbf{self-containedness-governed} training and decoding framework that aligns training with inference via \textit{No-Future} truncation and selects variable-size blocks using \textit{Future-Aware} \& \textit{No-Future} divergence.
    \item \textbf{Theoretical: }We relate NF-FA divergence to predictive consistency within committed blocks, providing intuition for why enforcing self-containedness mitigates errors from premature commitment.
\end{itemize}

\section{Preliminaries}
\label{sec:prelim}
We first define terminology used throughout the paper. Let $x_{1:L}$ denote a length-$L$ token sequence.
\begin{enumerate}[leftmargin=*,labelsep=0.0em,itemsep=0pt,parsep=0pt,topsep=0pt]
    \item \textbf{Prefix and commitment.} During block decoding, the model maintains a committed prefix $x_{1:p}$, ending at position $p$. \textit{Committing} means appending the newly generated block of tokens to this prefix. Once committed, these tokens are treated as fixed context and non-revisable.
    \item \textbf{Block budget and candidate boundary.} We restrict the size of a block to a fixed \emph{block budget} of length $W$ (e.g., $64$ tokens). We refer to $(p:q]$ as the budgeted block where $q=\min(p+W,L)$. $b \in (p:q]$ is a \textit{candidate block boundary} within the budgeted block.
\end{enumerate}
\subsection{Discrete Masked Diffusion}
\paragraph{Discrete masked diffusion training on the full token sequence for text.}
Let $\mathcal{V}$ be the vocabulary. Given an input token sequence  $x_{1:L}\in\mathcal{V}^L$ of length $L$ and diffusion step $t\in\{1,\dots,T\}$, a masked corruption process produces a noisy sequence as:
$x^{t} \sim q_t(\cdot \mid x)$,
where some positions are replaced by a special \texttt{[MASK]} token $\mathbf{M}$ based on a masking schedule~\cite{nie2025llada,MDLM}. A discrete dLLM parameterised by $\theta$ outputs per-position logits
$\ell_{\theta,i}(x^{t},t)\in\mathbb{R}^{|\mathcal{V}|}$ and categorical distributions:
\begin{equation}
\vspace{-0.3cm}
p_\theta(x_i \mid x^{t},t)=\mathrm{softmax}\big(\ell_{\theta,i}(x^{t},t)\big)[x_i].
\vspace{-0.0cm}
\label{eq:dlm_cond}
\end{equation}
The standard masked diffusion objective trains on full  sequences~\cite{nie2025llada,dream,diffusion_survey}, so the denoiser can attend to all unmasked tokens in $x^t_{1:L}$ when reconstructing masked positions:
\vspace{-0.1cm}
\begin{equation}
\mathcal{L}_\text{D}
=
-\mathbb{E}_{(x,t,x^{t})}
\left[
\frac{1}{t}\sum_{i=1}^{L}
\mathbf{1}[x_{i}^{t}=\text{M}] \log p_\theta(x_i \mid x^{t})
\right].
\label{eq:full_diff_loss}
\end{equation}

\paragraph{Blockwise semi-autoregressive decoding.}
Standard blockwise decoding partitions generation into fixed-size blocks~\cite{Arriola2025BlockDiff,fastdllm_v2,Li2025DAEDAL,nie2025llada,dllmcache,seeddiffusion}. Following the introduced notation, at each step, given a committed prefix ending at position $p$, the model allocates a fixed \textit{block budget} of length $W$ and decodes the entire budgeted block, commits them to the prefix, and then advances.

\subsection{Training-Inference Mismatch} \label{subsec:train_infer_mismatch}
A central difficulty in blockwise decoding for discrete diffusion language models is a mismatch between how the model is trained and how it is used at inference time.

\paragraph{Full-sequence training.}
Under the standard masked diffusion objective in Eq.~\eqref{eq:full_diff_loss}, a dLLM is trained on complete token sequences of length $L$. At each diffusion step $t$, predictions at every position are conditioned bidirectionally on the entire noisy sequence $\tilde{x}^t_{1:L}$, obtained by corrupting the clean input $x_{1:L}$~\cite{nie2025llada,dream}. For analysing blockwise generation, it is helpful to partition the sequence into three contiguous regions: the committed prefix $x_{1:p}$, a candidate next block $x_{p+1:q}$, and the remaining suffix $x_{q+1:L}$ (corresponds to unavailable future tokens for inference time)~\cite{Arriola2025BlockDiff}. During full-sequence training, predictions for tokens in $x_{p+1:q}$ can condition on \textit{both} the prefix and the suffix, whereas during blockwise inference the suffix becomes unavailable.

\paragraph{Blockwise inference.}
In contrast, blockwise decoding commits tokens incrementally: a boundary position $q$ determines the next committed block $(p:q]$. When committing this block, the model conditions only on information up to $q$, since tokens beyond $q$ have not been generated.

\paragraph{Resulting mismatch.}
As a result, the conditional distributions learned during training differ from those used during inference~\cite{sun2025blockwisesft}. Intuitively, a token may appear easy to denoise during training because the model can rely on future context from the suffix, but the same token may become ambiguous or unstable when that future information is removed at inference time (Figure~\ref{fig:committing_comparison}). Formally, for a token position $i \in [p+1:q]$ with corresponding masked sequence $\tilde{x}$, training uses full-sequence conditionals whereas blockwise decoding uses truncated conditionals:
\begin{equation}
\underbrace{p_\theta(x_i \mid \tilde{x}^{t}_{1:L}, t)}_{\text{full-sequence masked training}},\quad
\underbrace{p_\theta(x_i \mid \tilde{x}^{t}_{1:q}, t)}_{\text{blockwise decoding}},
\end{equation}
leading to the mismatch:
\begin{equation}
p_\theta(x_i \mid \tilde{x}^{t}_{1:L}, t)
\;\neq\;
p_\theta(x_i \mid \tilde{x}^{t}_{1:q}, t).
\label{eq:mismatch}
\end{equation}

\paragraph{Relation to prior blockwise training methods.}
Recent fixed-block alignment approaches~\citep{sun2025blockwisesft,Chen2025dParallel,fastdllm_v2,Arriola2025BlockDiff,blcokd2f,zhong2026beyond} also restrict future context during training to better match blockwise inference. However, these methods rely on predefined block boundaries, for example fixed block sizes or punctuation and line breaks. In contrast, our work establishes a self-containedness principle and focuses on decoding a block that is self-contained, enabling variable block sizes determined by the model's sensitivity to future context. A detailed comparison is provided in Section~\ref{sec:related} and Appendix~\ref{appendix:related_work}.

\section{Motivation: From Self-containedness to Dynamic Block Boundary}
\label{sec:objective}
Following the twin issues of training-inference mismatch and heuristic block boundaries, a natural question arises:

\textit{Given a budgeted block, which part of the tokens are ready to be committed without relying on future tokens to resolve?}

We argue that \textbf{self-containedness} should govern block commitment: \textbf{a block is self-contained if committing it does not change the model's predictive distributions when future context is revealed}.

\subsection{Formalising Self-containedness} \label{sec:formalising_self_containedness}
Following the defined notations in Section~\ref{sec:prelim}. Let $p$ denote the end position of the committed prefix.
Fix a block budget of length $W$ and let $q = \min(p+W,L)$ be the end position of the budgeted block $(p:q]$. Within this budget, we consider committing a shorter \textit{candidate block} $(p:b]$ at $b\in(p,q]$, leaving the remaining tokens $(b:q]$ as the \textit{within-block future}:
\begin{equation} 
\mathrm{Block}(p,b) = (p : b], \qquad
\mathrm{Future}(b) = (b : q].
\label{eq:block_construction}
\end{equation}
A candidate block ending at $b$ is self-contained only if predictions for tokens in $\mathrm{Block}(p,b)$ remain consistent when the future $\mathrm{Future}(b)$ is revealed. We formalise this using two conditioning semantics.

\begin{definition}[No-Future and Future-Aware conditionals] \label{def:nf_fa}
Let $p$ be the end of the committed prefix and let $q=\min(p+W,L)$ denote the end of the block budget of length $W$. For any candidate boundary $b$ with $p < b \le q$ and any $i \in (p:b]$, define:
\paragraph{1. No-Future (NF).}
The model predicts $x_i$ using context only up to $b$ (no access to tokens beyond $b$):
\begin{equation}
P^{\mathrm{NF}}_{\theta,i}(p,b)
=
p_\theta(x_i \mid x_{1:b}),
\qquad i \in(p:b].
\label{eq:nf_dist_def}
\end{equation}
\paragraph{2. Future-Aware (FA).}
The model predicts $x_i$ with access to the full budgeted block $x_{1:q}$ (including the future $x_{b+1:q}$):
\begin{equation}
P^{\mathrm{FA}}_{\theta,i}(p,b,q)
=
p_\theta(x_i \mid x_{1:q}),
\qquad i\in(p:b].
\label{eq:FA_dist_def}
\end{equation}
\end{definition}
\begin{figure}[!t]
    \centering
    \includegraphics[width=1\linewidth]{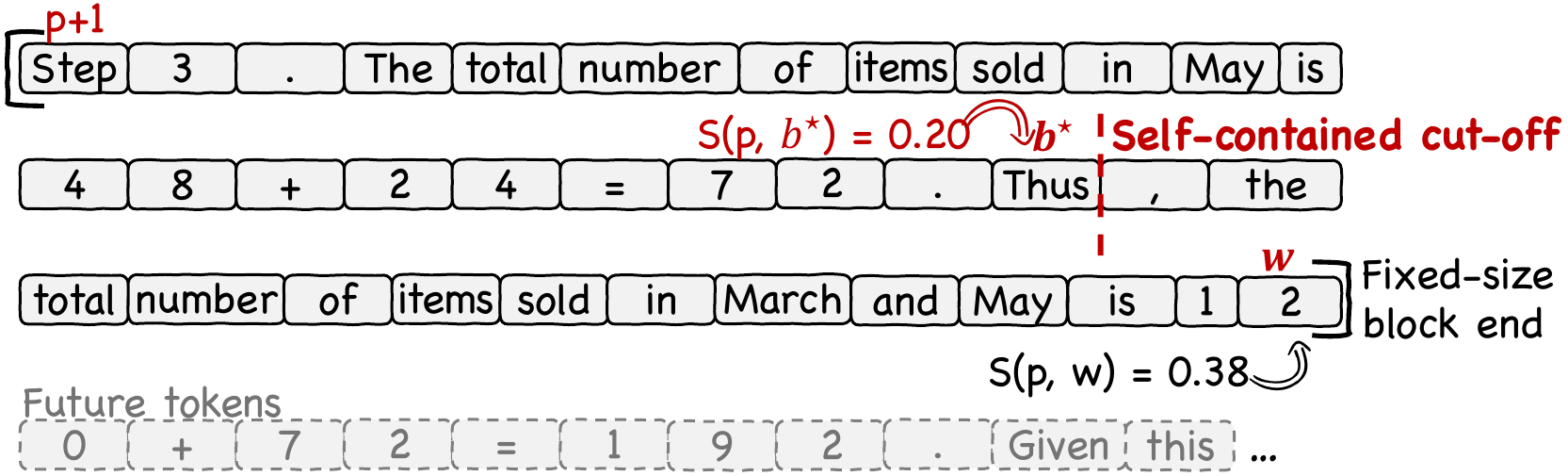}
    \caption{\textbf{Self-containedness guided block selection.} Fixed-size blocks can cut off mid-number with high future dependence, while VSB selects a boundary 
    that is more semantically complete.}
    \label{fig:score_illustration}
    \vspace{-0.3cm}
\end{figure}
\paragraph{Self-containedness Divergence.} Let $\mathcal{D}$ denote a divergence between categorical distributions, such as the KL divergence. We define the \textbf{self-containedness} of a block with a candidate boundary $b$ as the average per-token divergence between these predictive distributions:
\begin{equation}
S(p,b)
=
\frac{1}{b-p}
\sum\nolimits_{i=p+1}^{b}
\mathcal{D}\!\left(
P^{\mathrm{NF}}_{\theta,i}
\,\middle\|\,
P^{\mathrm{FA}}_{\theta,i}
\right),
\label{eq:self_contain_score}
\end{equation}
which tells on average how strongly predictions for tokens $x_{p+1:b}$ depend on future context beyond $b$. Intuitively, the lower the divergence, the more self-contained the block is.
\begin{definition}[$\varepsilon$-self-contained block] \label{def:epsilon_self_contained}
A block $\mathrm{Block}(p,b)$ is \textit{$\varepsilon$-self-contained} if $S(p,b) \le \varepsilon$.
\end{definition}
\noindent\textbf{Implications.}
Since $S(p,b)$ measures the average NF-FA divergence within a block, Definition~\ref{def:epsilon_self_contained} implies that future tokens add little information for predicting tokens in $(p\!:\!b]$. Figure~\ref{fig:score_illustration} shows that self-containedness prefers semantically complete boundaries over fixed-size blocks, while Section~\ref{sec:exp} Figure~\ref{fig:self-containedness_density} compares their distributions.

\subsection{Self-containedness and Predictive Consistency}
Committing a block at boundary $b$ is problematic when predictions under \textit{No-Future} and \textit{Future-Aware} conditioning diverge, as this indicates that the model's decisions for tokens in $(p:b]$ still depend on the future context $(b:q]$. The self-containedness divergence $S(p,b)$ quantifies this consistency by measuring the average discrepancy between NF and FA predictive distributions. When $S(p,b)$ is small, the model's predictions for tokens in $(p:b]$ remain stable even when future tokens are revealed, as formalised below.

\begin{theorem}[Predictive consistency under self-containedness]
\label{thm:consistency}
Let $\mathrm{Block}(p,b^\star)$ be an $\varepsilon$-self-contained block, where $P^{\mathrm{NF}}$ and $P^{\mathrm{FA}}$ denote the predictive distributions under No-Future (NF) and Future-Aware (FA) conditioning, respectively.
Assume that the divergence $\mathcal{D}$ upper-bounds total variation distance $\delta_{\mathrm{TV}}(P,Q)$ up to a constant. Then the average NF-FA predictive discrepancy within the block, measured in total variation is bounded by:
$$
\frac{1}{b^\star-p}\sum\nolimits_{i=p+1}^{b^\star}
\delta_{\mathrm{TV}}\!\left(P_i^{\mathrm{NF}}, P_i^{\mathrm{FA}}\right)
\le\
\mathcal{O}(\sqrt{\varepsilon}).
$$
Equivalently, revealing future context alters token probabilities within the block by at most $\mathcal{O}(\sqrt{\varepsilon})$ on average.
\end{theorem}
\paragraph{Interpretation.} Theorem~\ref{thm:consistency} implies that when the self-containedness divergence $S(p,b)$ is small, revealing future context induces only a small change in the model's predictive distributions within the block, on average. The proof is provided in Appendix~\ref{proof:self_contained_commit}.

\subsection{Why Fixed Block Size Is Structurally Limiting}
Under the self-containedness principle, whether a block can be committed depends on how strongly its predictions rely on future context. 
As a result, the commit length that satisfies $S(p,b)\le\varepsilon$ can vary substantially across prefixes: some regions admit long self-contained blocks, while others require shorter commits.

A fixed-size decoding strategy, however, commits a constant block size $W$ regardless of content. Whenever this fixed length exceeds what self-containedness permits at a given prefix, the committed block violates $S(p,b)\le\varepsilon$ and becomes future-dependent.
Formally:

\begin{proposition}[Characterisation of fixed-size self-containedness]
\label{prop:fixed_length_characterisation}
Fix a block size $W>0$. A fixed-size commitment strategy with block size $W$
satisfies $S(p,b)\le\varepsilon$ for all committed blocks if and only if:
$
S(p,p+W)\le\varepsilon
$
for every prefix position $p$ at which a block of size $W$ is committed.
\end{proposition}
\vspace{-0.4cm}
\paragraph{Interpretation.}
Proposition~\ref{prop:fixed_length_characterisation} shows why fixed-size commitment is restrictive under self-containedness: since the same $W$ is applied everywhere, self-containedness must hold uniformly across all prefixes. When future dependence varies across the sequence, no single fixed block size can satisfy this condition consistently, motivating adaptive boundary selection within a fixed block budget $(p:q]$. The proof is in Appendix~\ref{proof:fixed_length_characterisation}.

\begin{figure*}[t!]
    \centering
    \includegraphics[width=0.8\linewidth]{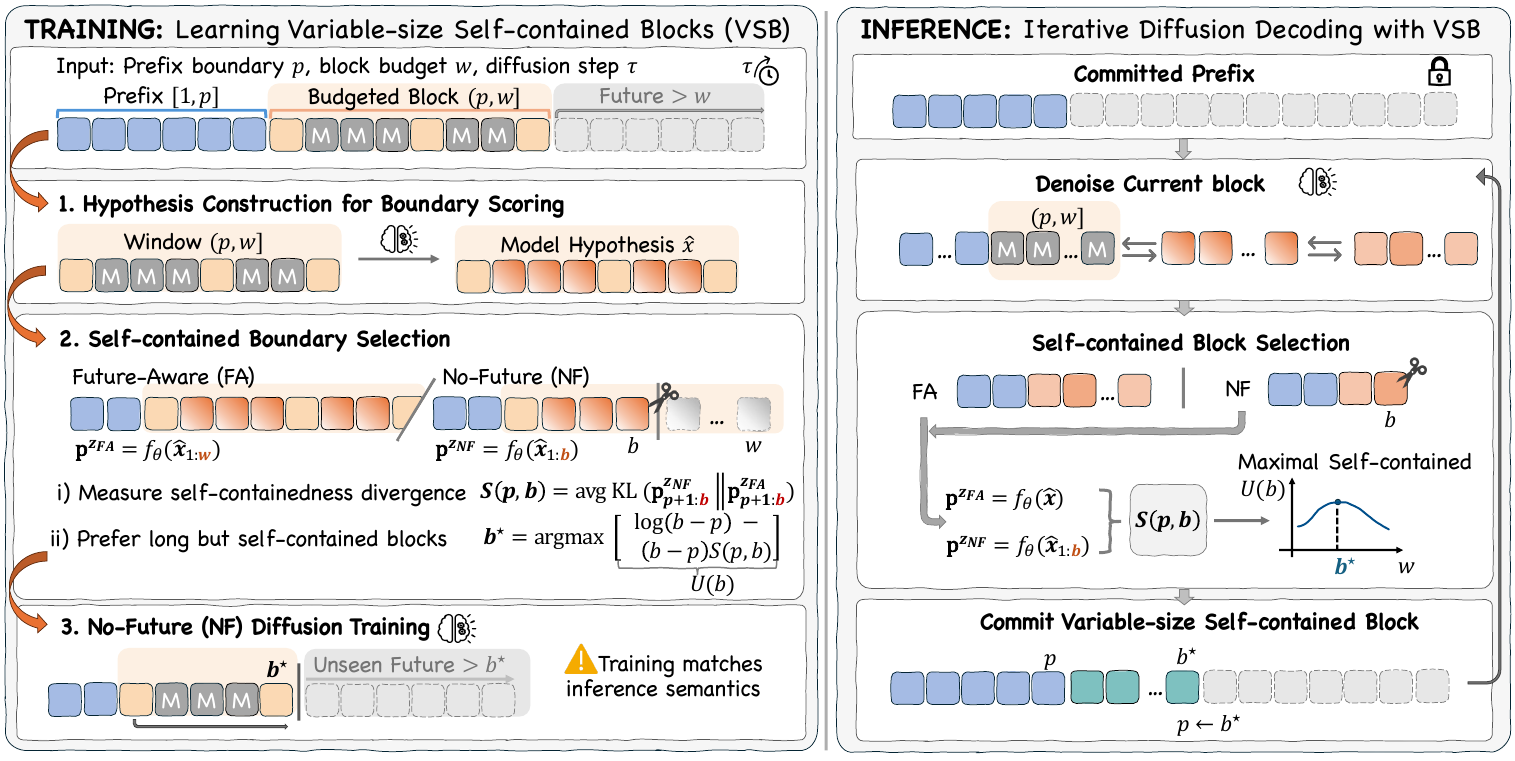}
    \caption{\textbf{Overview of VSB.} During training, candidate boundaries in the budgeted block are scored by self-containedness using \textit{No-Future} and \textit{Future-Aware} conditionals to encourage long, self-contained blocks. At inference, diffusion decoding commits the block that best balances length and self-containedness, yielding adaptive, semantically aligned blocks.}
    \label{fig:vsb}
    \vspace{-0.5cm}
\end{figure*}
\section{Method: Variable-size Self-contained Blocks}
\label{sec:method}
We now present Variable-size Self-contained Blocks (VSB), a training and decoding framework that operationalises the self-containedness principle for blockwise generation in discrete dLLMs. At each step, VSB selects a commitment boundary within a fixed block budget by explicitly measuring how sensitive the block's predictions are to future context. We first describe the core components \textbf{shared by training and inference}, and then give \textbf{phase-specific details}. Figure~\ref{fig:vsb} provides an overview.

\subsection{Primary Mechanisms of VSB}
\paragraph{4.1.1. Block construction.}
We follow the terminology in Section~\ref{sec:prelim}: given a committed prefix ending at $p$, we define a budgeted block $(p,q]$ of length $W$ ending at $q=\min(p+W,L)$. Our goal is to select boundary $b \in (p,q]$.

\paragraph{4.1.2. Hypothesis construction for boundary scoring.} \label{sec:hypothesis}
Self-containedness compares the model's predictions for tokens $(p,b]$ with and without access to future tokens $(b,q]$ (Definition~\ref{def:nf_fa}). Since the true future tokens are unavailable at inference time, VSB performs this comparison on a fixed model-generated \textit{hypothesis} $\hat{x}_{1:q}$, i.e., the model's current token estimate over the budgeted block $(p,q]$. All NF and FA distributions are computed on the same hypothesis.
This hypothesis serves two purposes: it matches the model's decoding time prediction conditioning and ensures that differences in self-containedness across boundaries arise only from access to future tokens, not changes in decoding state. Formally, given a committed prefix $x_{1:p}$ and window end $q$, the hypothesis $\hat{x}_{1:q}$ is obtained by (partially) resolving the masked window $(p,q]$ conditioned on $x_{1:p}$:

\textbf{Training.} The prefix $x_{1:p}$ is taken from the ground-truth sequence, and the hypothesis is constructed by refining the masked positions in the active budgeted block using the model's predictions:
\begin{equation}
\hat{x}_{p+1:q}
=
\arg\max p_\theta(\cdot \mid x_{1:p}, \text{M}_{p+1:q}).
\label{eq:tr_hypothesis}
\end{equation}

\textbf{Inference.} The prefix consists of previously committed tokens $x_{1:p}$, and the hypothesis is taken from an intermediate snapshot of the diffusion process:
\begin{equation}
\hat{x}_{p+1:q}
=
\tilde{x}^{\,t_\text{hyp}}_{p+1:q},
\end{equation}
where $\tilde{x}^{\,t_\text{hyp}}$ is the partially decoded sequence at the hypothesis step $t_\text{hyp}$ used for boundary scoring (e.g., final step). Figure~\ref{fig:vsb} illustrates this construction.

\paragraph{4.1.3. Length-aware self-contained boundary selection.} \label{sec:self-contained_boundary_selection_criteria}
Given a hypothesis $\hat{x}_{1:q}$, VSB measures how strongly predictions in a candidate block depend on future context. For a candidate boundary $b \in \mathcal{B}(p,q)$, let $\text{p}^{\mathrm{FA}}$ and $\text{p}^{\mathrm{NF}}$ denote the FA and NF predictive distributions from Definition~\ref{def:nf_fa}. The self-containedness divergence of $(p,b]$ can be obtained via Eq.~\eqref{eq:self_contain_score}. A natural idea is to commit the \textbf{longest block} that satisfies an explicit self-containedness threshold:
\begin{equation}
\max_{b}
\;\; b-p
\quad \text{s.t.} \quad
p < b \le q,
\;\;
S(p,b) \le \varepsilon .
\label{eq:hard_decision_rule}
\end{equation}
In practice, this hard constraint in Eq.~\eqref{eq:hard_decision_rule} can be brittle: small fluctuations in $S(p,b)$ may cause abrupt changes in the selected boundary, or yield no valid $b$ given the choice of $\varepsilon$. Hence, to obtain a \textbf{smoother decision rule without thresholds}, we relax Eq.~\eqref{eq:hard_decision_rule} by trading-off block length against a length-aware self-containedness divergence $S_\text{L-A}$, and select the boundary that maximises this trade-off:
\begin{equation}
\begin{split}
S_{\text{L-A}} &= 
\log(b-p)
-
(b-p)\, S(p,b)
, \\
b^\star
&=
\arg\max_{b \in \mathcal{B}(p,q)}
\Big[
S_{\text{L-A}}
\Big],
\end{split}
\label{eq:boundary_selection_rule}
\end{equation}
where $\mathcal{B}(p,q)=\{b \mid p < b \le q\}$ denotes the candidate boundaries within the block budget.
\paragraph{Intuition.}
Since $S(p,b)$ is an \textit{average} per-token future sensitivity inside $(p,b]$, multiplying by $(b-p)$ yields the \textit{total future-dependence mass} of the block. This matches the commitment risk: once committed, all tokens in the block are frozen, so we penalise how much future-dependent content we are locking in, not merely whether the average token is stable. \textbf{Intuitively}, the first term favors longer commits (diminishing returns via $\log$), while the second discourages committing blocks whose total future dependence is large. A detailed illustration of the relationship between $b,S$ and $S_\text{L-A}$ is provided in Figure~\ref{fig:length_aware_vs_avg}.

\subsection{VSB Training with \textit{No-Future }Alignment}

\textbf{Training mirrors inference:} for each example, we sample a prefix end position $p$ uniformly, define the block budget end $q = \min(p+W,L)$ with length $W$, construct the hypothesis $\hat{x}_{1:q}$ (Eq.~\eqref{eq:tr_hypothesis}), and select the optimal boundary $b^\star$ via Eq.~\eqref{eq:boundary_selection_rule}. We then apply \textit{No-Future} diffusion training as:
\begin{itemize}[noitemsep, topsep=0pt]
    \item Tokens $(p,b^\star]$ are masked according to the standard diffusion schedule.
    \item Truncate the model input to length $b^\star$ to match the inference-time absence of tokens beyond $b^\star$.
\end{itemize}

Let $\mathcal{M}\subseteq[p+1:b^\star]$ denote the masked positions inside the committed block at step $t$ of the noisy sequence $\tilde{x}$, and $\alpha(p,b^\star) = \frac{b^\star-p}{w}$ be a length-dependent weighting factor that discourages degenerate short blocks (e.g., 1 token). The \textit{No-Future} diffusion loss is defined as:
\begin{equation}
\mathcal{L}_{\mathrm{NF}}
=
\alpha(p,b^\star)
\cdot
\frac{1}{|\mathcal{M}|}
\sum_{i \in \mathcal{M}}
-\log p_\theta(x_i \mid \tilde{x}_{1:b^\star}^t).
\label{eq:final_nf_diff_loss}
\end{equation}

The training process is shown in Algorithm~\ref{alg:training} in Appendix~\ref{appendix:algorithms}.

\subsection{VSB Inference Procedure}
At inference time, VSB alternates between diffusion denoising within a budgeted block and selecting a self-contained commitment boundary, as shown in Algorithm~\ref{alg:vsb_inference}. Given a committed prefix $x_{1:p}$, after denoising the budgeted block $(p,q]$, a hypothesis $\hat{x}_{1:q}$ is constructed following Section~\ref{sec:hypothesis}. The optimal block $\mathrm{Block}(p,b^\star)$ is then selected via Eq.~\eqref{eq:boundary_selection_rule} and committed to the prefix:
$$
x_{1:p} \leftarrow \hat{x}_{1:b^\star},
\qquad
p \leftarrow b^\star.
$$
The budgeted block is then advanced and the procedure repeats until generation completes. Since the same training and inference conditioning and hypothesis construction are used, \textbf{VSB introduces no training-inference mismatch}.

\subsection{Boundary Search Acceleration} \label{sec:coarse-fine}
VSB's boundary selection (Algorithm~\ref{alg:vsb_inference}) can be implemented as a full scan over the $W$ candidate boundaries. While this is conceptually simple, it is inefficient in practice. We \textbf{accelerate boundary selection} using a \textbf{coarse-to-fine search} (Algorithm~\ref{alg:vsb_coarse_to_fine}): a coarse pass evaluates boundaries at a stride $s_c$ to locate a high-scoring region under Eq.~\eqref{eq:boundary_selection_rule}, and a fine pass then refines the choice within a local neighborhood at stride $s_f$. This \textbf{preserves self-containedness} while reducing boundary evaluation from $O(W)$ to:
\begin{equation}
O\!\left(\frac{W}{s_c} + \frac{s_c}{s_f}\right).
\end{equation}
\textbf{Practical impact.} Table~\ref{tab:vsb_c2f_ablation} shows this acceleration preserves accuracy while providing consistent speedups, making self-contained boundary selection practical at inference.

\begin{table*}[t]
\centering
\small
\setlength{\tabcolsep}{6pt}
\resizebox{\linewidth}{!}{
\begin{tabular}{r|lll|ll|ll}
\toprule
\multirow{3}{*}{\textbf{Method}}
& \multicolumn{3}{c|}{\textit{Math \& Science}}
& \multicolumn{2}{c|}{\textit{Code}}
& \multicolumn{2}{c}{\textit{General Knowledge}} \\
& \textbf{GSM8K}
& \textbf{MATH500}
& \textbf{GPQA-Diamond}
& \textbf{HumanEval}
& \textbf{MBPP}
& \textbf{MMLU\_Gen}
& \textbf{HellaSwag} \\
& \small(512, 0)
& \small(512, 4)
& \small(64, 0)
& \small(512, 0)
& \small(256, 3)
& \small(3, 0)
& \small(10, 0) \\
\midrule
LLaDA-8B & 76.64  & 35.80 & 26.52 & 37.20 & 35.20 & 62.49 & 72.74 \\
dLLM-Var         & 73.92 & 34.00 & 23.73 & 34.76 & \textcolor{purple}{\textbf{40.20}} & 61.71 & 71.48 \\
Blockwise-SFT    & 75.06 & 37.80 & 26.26 & 37.20 & 35.60 & 62.81 & 73.10 \\
D2F              & 74.90 & 36.00 & 20.20 & 40.24 & 34.60  & 62.71 & 73.33\\
AdaBlock         & 75.66 & 38.20 & \textcolor{purple}{\textbf{28.79}} &  \textcolor{purple}{\textbf{45.12}}  & 37.60 & \textcolor{teal}{\underline{62.98}} & 73.94 \\
DAEDAL           & \textcolor{teal}{\underline{80.74}} & \textcolor{teal}{\underline{38.60}} & \textcolor{teal}{\underline{28.28}} & 39.02 & 36.00 & 62.58 & \textcolor{teal}{\underline{75.60}} \\
\midrule
\rowcolor{academicBlue!10}
\textbf{VSB} 
 w/ LLaDA-8B
& \textcolor{purple}{\textbf{81.12}}$_{+4.48\%}$ & \textcolor{purple}{\textbf{39.00}}$_{+3.20\%}$ & \textcolor{purple}{\textbf{28.79}}$_{+2.27\%}$ & \textcolor{teal}{\underline{43.29}}$_{+6.09\%}$ & \textcolor{teal}{\underline{38.60}}$_{+3.40\%}$ & \textcolor{purple}{\textbf{63.11}}$_{+0.62\%}$ & \textcolor{purple}{\textbf{76.68}}$_{+3.94\%}$ \\
\midrule\midrule
LLaDA-1.5        & 79.38 & 37.00 & 27.78 & 36.59 & 35.60 & 62.52 & 72.34 \\
\rowcolor{academicBlue!10}
\textbf{VSB} w/ LLaDA-1.5
& 83.40$_{+4.02\%}$ & 39.60$_{+2.60\%}$ & 28.28$_{+0.5\%}$ & 46.95$_{+10.36\%}$ & 39.80$_{+4.20\%}$ & 62.94$_{+0.42\%}$ & 76.93$_{+4.59\%}$ \\
\bottomrule
\end{tabular}
}
\caption{
Overall performance of VSB against baselines across diverse benchmarks. Scores are reported as \textbf{Accuracy/Pass@1 (\%)}. Numbers in parentheses denote \textbf{token budget} and \textbf{number of shots}. For fair comparison, best among the \textbf{LLaDA-8B-Instruct based methods} are  \textcolor{purple}{\textbf{bolded}} and Runnerups are \textcolor{teal}{\underline{underlined}}, respectively. The \textbf{relative improvements} of \textbf{VSB against the backbone} are labelled.}
\label{tab:main_results_transposed}
\vspace{-0.6cm}
\end{table*}
\section{Experiments} \label{sec:exp}
\paragraph{Datasets and Evaluation.}
We evaluate on standard LLM benchmarks spanning mathematical \& scientific reasoning (GSM8K~\cite{gsm8k}, MATH500~\cite{math500}, GPQA-Diamond~\cite{GPQA-Diamond}), code generation (HumanEval~\cite{humaneval}, MBPP~\cite{mbpp}), and general knowledge (MMLU\_Generative~\cite{mmlu}, HellaSwag~\cite{hellaswag}). We report accuracy/pass@1 (\%), with more details provided in Appendix~\ref{appendix:Extended_dataset_description}.

\paragraph{Baselines.}
We compare VSB against blockwise-training and semi-AR decoding methods: \textbf{(1)} fixed-size block decoding of LLaDA backbones ({\fontfamily{qcr}\selectfont LLaDA-8B-Instruct} (short for LLaDA-8B)~\cite{nie2025llada} and {\fontfamily{qcr}\selectfont LLaDA-1.5}~\cite{llada1_5}); \textbf{(2)} {\fontfamily{qcr}\selectfont dLLM-Var}~\cite{dllmvar} with native variable-length generation via \texttt{[EOS]} prediction; \textbf{(3)} {\fontfamily{qcr}\selectfont Blockwise-SFT}~\cite{sun2025blockwisesft} aligns supervised fine-tuning with fixed blockwise-decoding; \textbf{(4)} {\fontfamily{qcr}\selectfont D2F}~\cite{blcokd2f} distills a student dLLM with blockwise causal attention; and adaptive block methods \textbf{(5)} {\fontfamily{qcr}\selectfont AdaBlock}~\cite{lu2025adablock} and \textbf{(6)} {\fontfamily{qcr}\selectfont DAEDAL}~\cite{Li2025DAEDAL}.

\paragraph{Implementations.} We use LLaDA-8B-Instruct as the backbone for all training-based and post-hoc baselines, and report VSB results with both LLaDA-1.5 and LLaDA-8B-Instruct. Dataset-specific token budgets, block lengths, and few-shot counts follow standard settings from prior work~\cite{dllm_repo,nie2025llada,lu2025adablock} and are listed in Table~\ref{tab:main_results_transposed}. We sweep the block budget length $W \in \{32,64,96\}$ (default $W=96$) during training. For fair comparison, we match the inference-time budget $W$ to the fixed block length used by each baseline. We use the training split of {\fontfamily{qcr}\selectfont openai/gsm8k}~\cite{gsm8k} for supervised fine-tuning. Full details are in Appendix~\ref{appendix:implementation_details}.

\subsection{Overall Performance}
\textbf{VSB consistently improves over its backbones and is competitive with prior blockwise decoding and training approaches.} Table~\ref{tab:main_results_transposed} shows that on mathematical and scientific reasoning tasks, VSB achieves the strongest or tied-best results among the base methods, suggesting that self-contained boundary selection reduces premature decisions in multi-step reasoning. On code generation, VSB yields larger gains (6.0\% to 10.0\% over corresponding backbones on HumanEval) -- \textbf{despite not involving code-specific training} -- consistent with the need for variable-length semantic units in code as shown in other variable block-length methods like AdaBlock and DAEDAL. And on general knowledge tasks, improvements are smaller but consistent, indicating that VSB does not trade off broad language modelling capability. We provide a qualitative case study comparing VSB to fixed-size decoding in Appendix~\ref{appendix:case_study1}.

\subsection{Effective Boundary Selection via Self-Containedness} \label{sec:length_aware_vs_avg}
\textbf{VSB selects block boundaries where the decoded text is self-contained and semantically complete, rather than at fixed positions or heuristic cutoffs.}
Figure~\ref{fig:length_aware_vs_avg} shows how self-containedness is used to select a block boundary for a fixed prefix ending at position $p$ (here $p{=}994$). We evaluate candidate boundaries $b \in (p, p{+}W]$ within a fixed budget $W{=}64$, where each point in the top panels corresponds to committing the tokens $x_{p+1:b}$. The bottom panel displays the decoded tokens. The selected cutoff $b^*$ aligns with a semantically complete and self-contained text span. This boundary follows directly from the statistics above. In Figure~4b), $b^*$ is the maximum of the length aware trade off, indicating the point where the block is long enough to be meaningful while its predictions are already stable without future context. Tokens beyond $b^*$ show higher divergence, meaning their content still depends on additional continuation from the next block, clearly reflected from the incomplete equation in the decoded tokens. We highlight that using the raw self-containedness divergence (Figure~4c) alone is not sufficient. Simply choosing its minimum may result in very short blocks, often collapsing to a single token. The length aware objective avoids this by balancing divergence and block length, allowing VSB to select coherent boundaries without extra thresholds (Figure~4a). A full decoded token case study is provided in Appendix~\ref{appendix:case_study2}.
\subsection{VSB Decoding vs. Fixed-size Block Decoding} \label{sec:exp_fix_decoding}
Beyond training,\textbf{ VSB can be applied as a post-hoc decoding strategy}, enabling direct comparison with fixed-size block decoding on pretrained backbones.
\textbf{Replacing fixed-size commitment with VSB yields consistent gains.} As shown in Figure~\ref{fig:fixed_vs_vsb_decoding}, on Math500 and MBPP, applying VSB decoding instead of fixed-size decoding (block length 64) consistently improves accuracy, with gains ranging from 1.4\% to 4.4\%. Since each pair of bars differs only in the decoding strategy, these improvements directly indicate the efficacy of our self-contained boundary selection. The results suggest that committing an entire fixed-length block can be suboptimal, while committing the self-contained prefixes yields more reliable conditioning.
\begin{figure}[t!]
    \centering
    \includegraphics[width=1\linewidth]{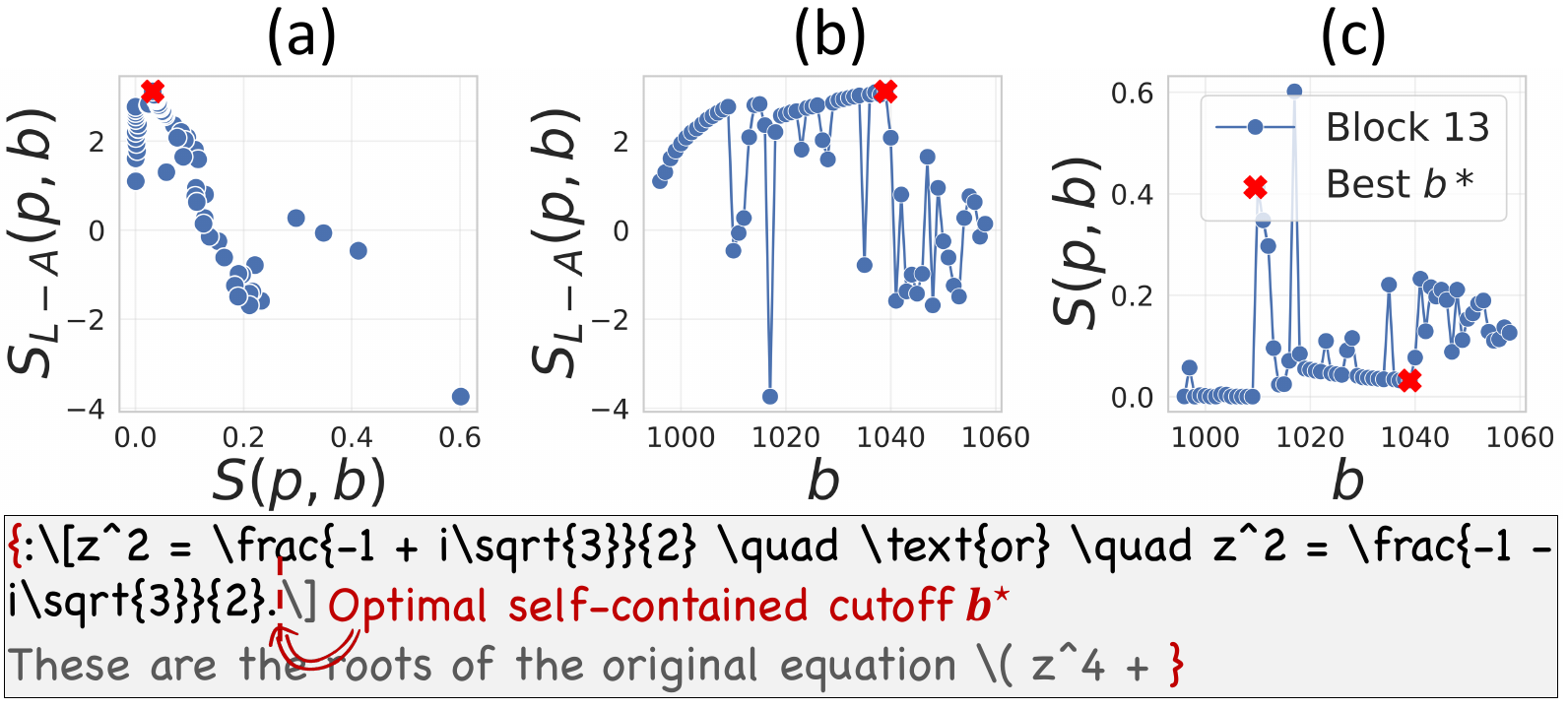}
    \caption{$S(p,b)$ vs. length-aware $S_{\mathrm{L-A}}(p,b)$ across a budgeted block (block $13$ w/ length $W=64$) starting at prefix $p = 994$.}
    \label{fig:length_aware_vs_avg}
    \vspace{-0.4cm}
\end{figure}

\begin{figure}[t!]
    \centering
    \includegraphics[width=1\linewidth]{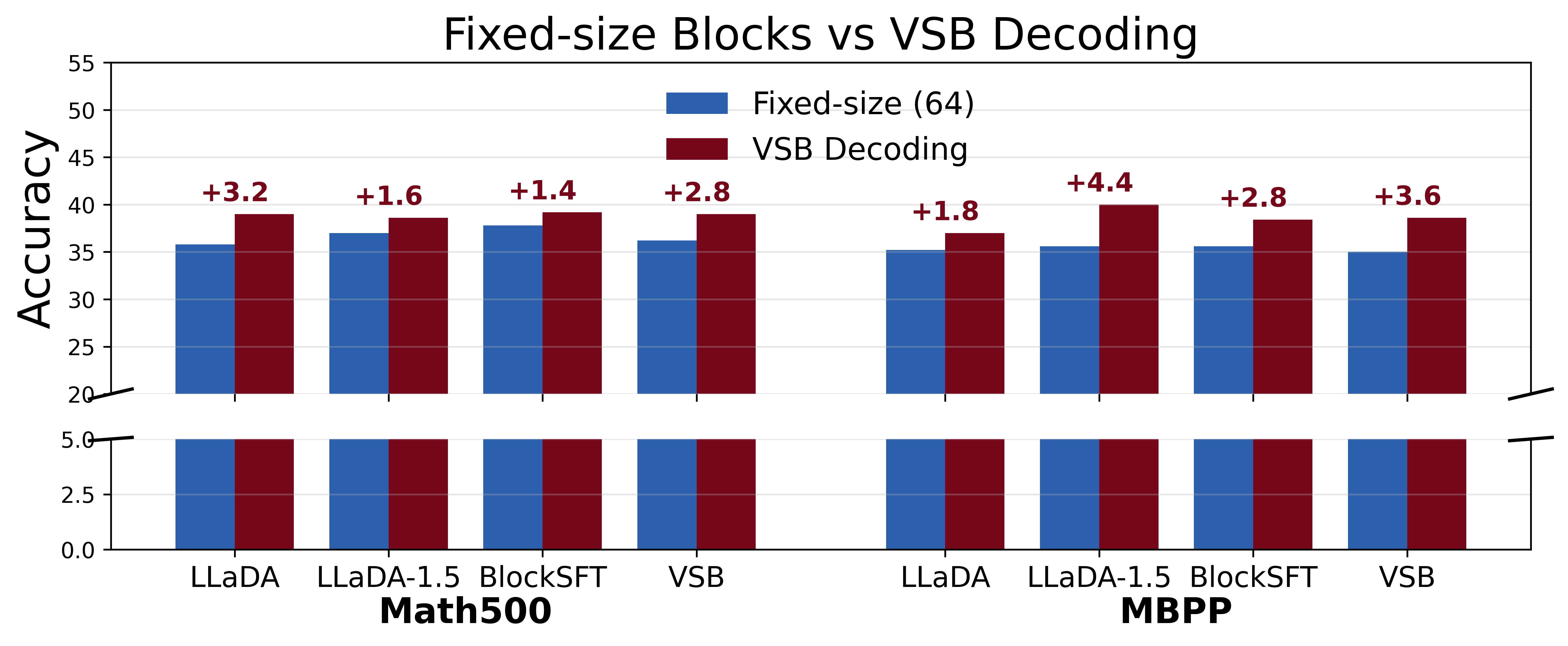}
    \caption{Fixed-size (block length $= 64$) vs. VSB decoding.}
    \label{fig:fixed_vs_vsb_decoding}
    \vspace{-0.6cm}
\end{figure}

\textbf{VSB decoding shifts commitment toward blocks with weaker dependence on future context.}
Figure~\ref{fig:self-containedness_density} shows the distribution of self-containedness divergence $S(p,b)$ for committed blocks under different decoding strategies. On GSM8K, fixed-size decoding yields a broad distribution with higher divergence, while VSB concentrates at lower $S(p,b)$ values. This is consistent with our objective to commit only when predictions are less sensitive to future context. On HumanEval, applying VSB decoding to the LLaDA backbone similarly shifts the distribution toward lower divergence and improves accuracy. This corroborates the performance increase in Figure~\ref{fig:fixed_vs_vsb_decoding}, demonstrating VSB as an effective decoding strategy beyond its training benefits.
\begin{figure}[!t]
    \centering
    \includegraphics[width=0.49\linewidth]{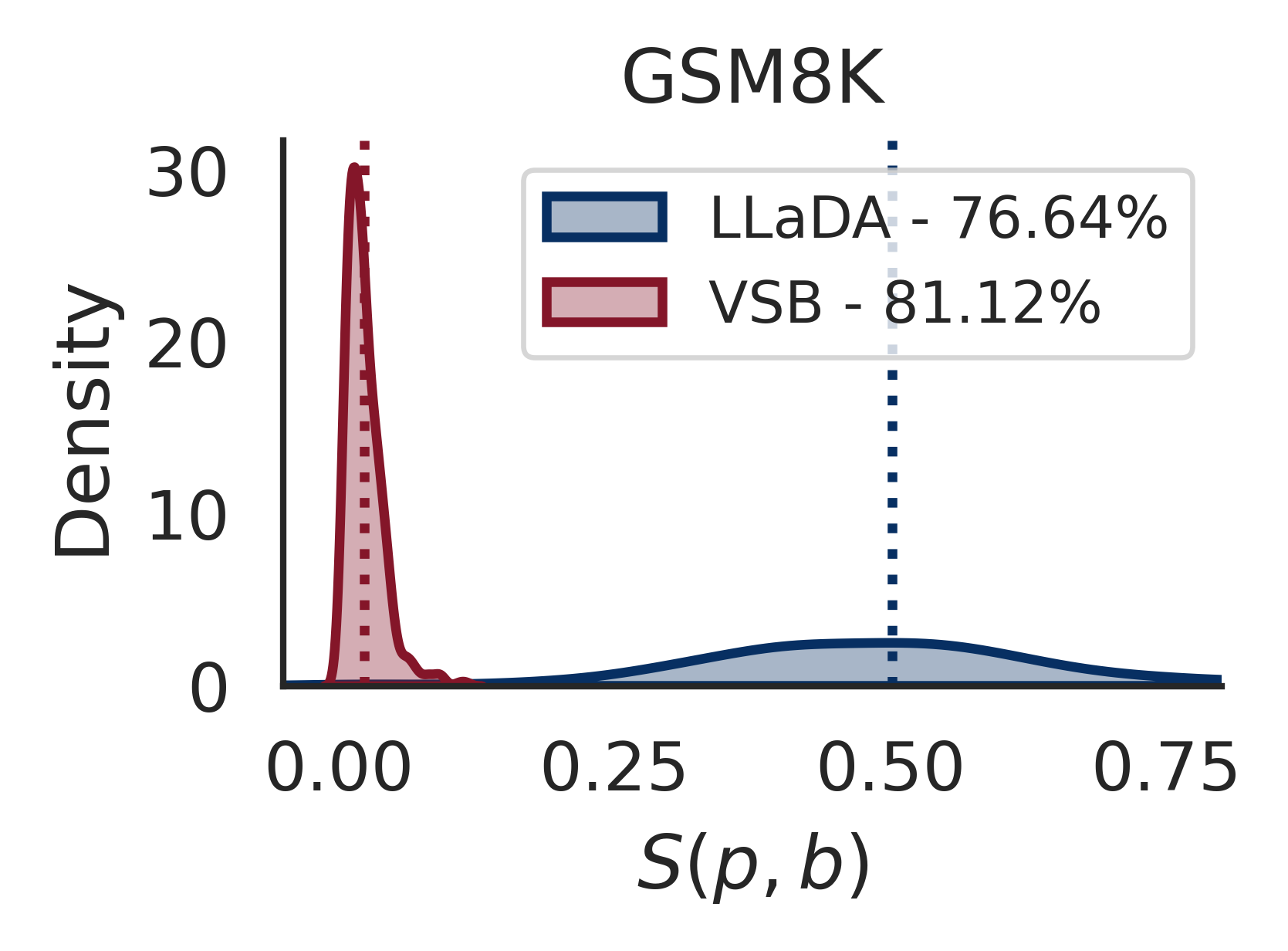}
    \hfill
    \includegraphics[width=0.49\linewidth]{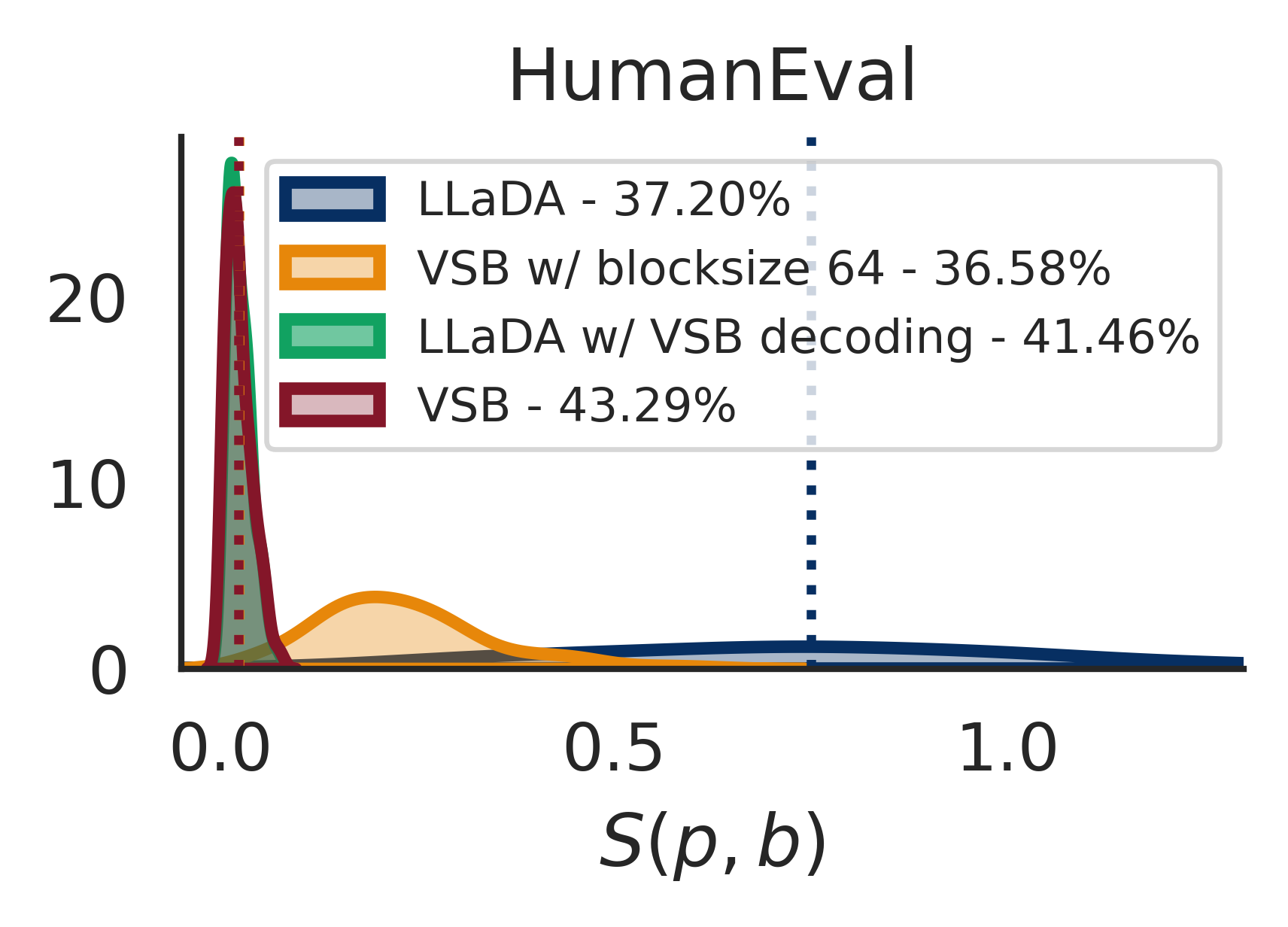}
    \vspace{-0.4cm}
    \caption{Self-containedness divergence distribution comparison.}
    \label{fig:self-containedness_density}
\end{figure}

\subsection{Confidence Threshold Decoding and Cache}
\textbf{VSB is compatible with confidence-threshold decoding and caching.}
As shown in Table~\ref{tab:dynamic_samp_cache}, VSB supports confidence-threshold decoding for speedups~\cite{fastdllm,dimple}, allowing denoising to stop once a confidence threshold (e.g., $\tau_{\text{conf}}=0.9$) is reached instead of using  a fixed schedule. Additionally, VSB can cache generated hidden states to reduce redundant computation. However, in the current LLaDA implementation, KV prefix caching can introduce a causal attention bias that affects NF/FA comparisons~\cite{fastdllm}. To preserve correct semantics of self-containedness scoring, we compute \textit{Future-Aware} and \textit{No-Future} distributions without caching, and instead reuse intermediate hidden states from the \textit{Future-Aware} forward with our coarse-to-fine search to accelerate boundary scoring. See Appendix~\ref{appendix:limitations} for details.
\begin{table}[!t]
\centering
\small
\setlength{\tabcolsep}{6pt}
\renewcommand{\arraystretch}{1.2}
\resizebox{\linewidth}{!}{
\begin{tabular}{l|cc|cc|cc}
\toprule
\multirow{2}{*}{\textbf{Method}}
& \multicolumn{2}{c|}{\textbf{GSM8K}}
& \multicolumn{2}{c|}{\textbf{GPQA-Diamond}}
& \multicolumn{2}{c}{\textbf{HellaSwag}} \\
& Acc. & Speedup
& Acc. & Speedup
& Acc. & Speedup \\
\midrule
LLaDA-8B
& 76.64 & --
& 26.52 & --
& 72.74 & --  \\
VSB
& \textbf{81.12}  & 1.0$\times$
& \textbf{28.79} & 1.0$\times$
& \textbf{76.68} & 1.0$\times$\\
w/ Cache
& \underline{80.36} & 1.2$\times$
& \underline{28.28} & 1.2$\times$
& 75.23 & 1.3$\times$\\
w/ Conf. Thresh.
& 77.48 & 2.7$\times$
& 27.27 & 2.5$\times$
& \underline{75.97} & 2.2$\times$ \\
\bottomrule
\end{tabular}}
\caption{VSB speedups with caching and confidence threshold.}
\label{tab:dynamic_samp_cache}
\vspace{-0.7cm}
\end{table}

\subsection{Hyperparameter Sensitivity}
VSB introduces a small number of hyperparameters, including the block budget $W$ 
. Shown in Figure~\ref{fig:block_size_accuracy}, \textbf{VSB consistently outperforms fixed-size decoding across all $W$} on GSM8K and HumanEval, with the largest gains at small to medium block sizes. Although larger blocks provide more context, committing long blocks can cause unstable predictions. In contrast, VSB commits only self-contained tokens, resulting in shorter and variable average block lengths (as \textcolor{purple}{annotated}). This indicates that the improvements stem from boundary selection rather than increased within-block context, avoiding accumulation of future-dependent tokens.

\begin{figure}[!t]
    \centering
    \includegraphics[width=0.48\linewidth]{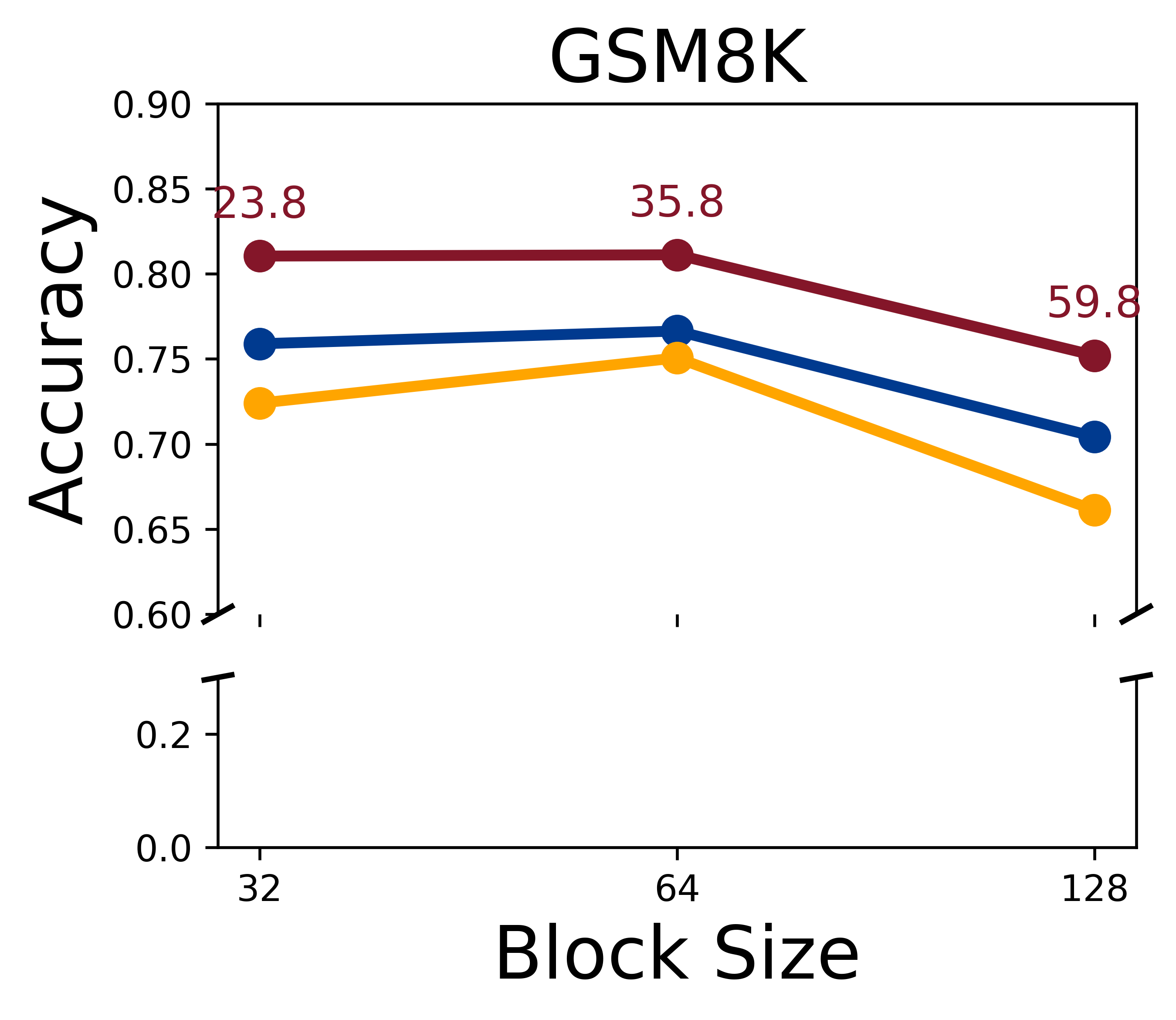}
    \hfill
    \includegraphics[width=0.48\linewidth]{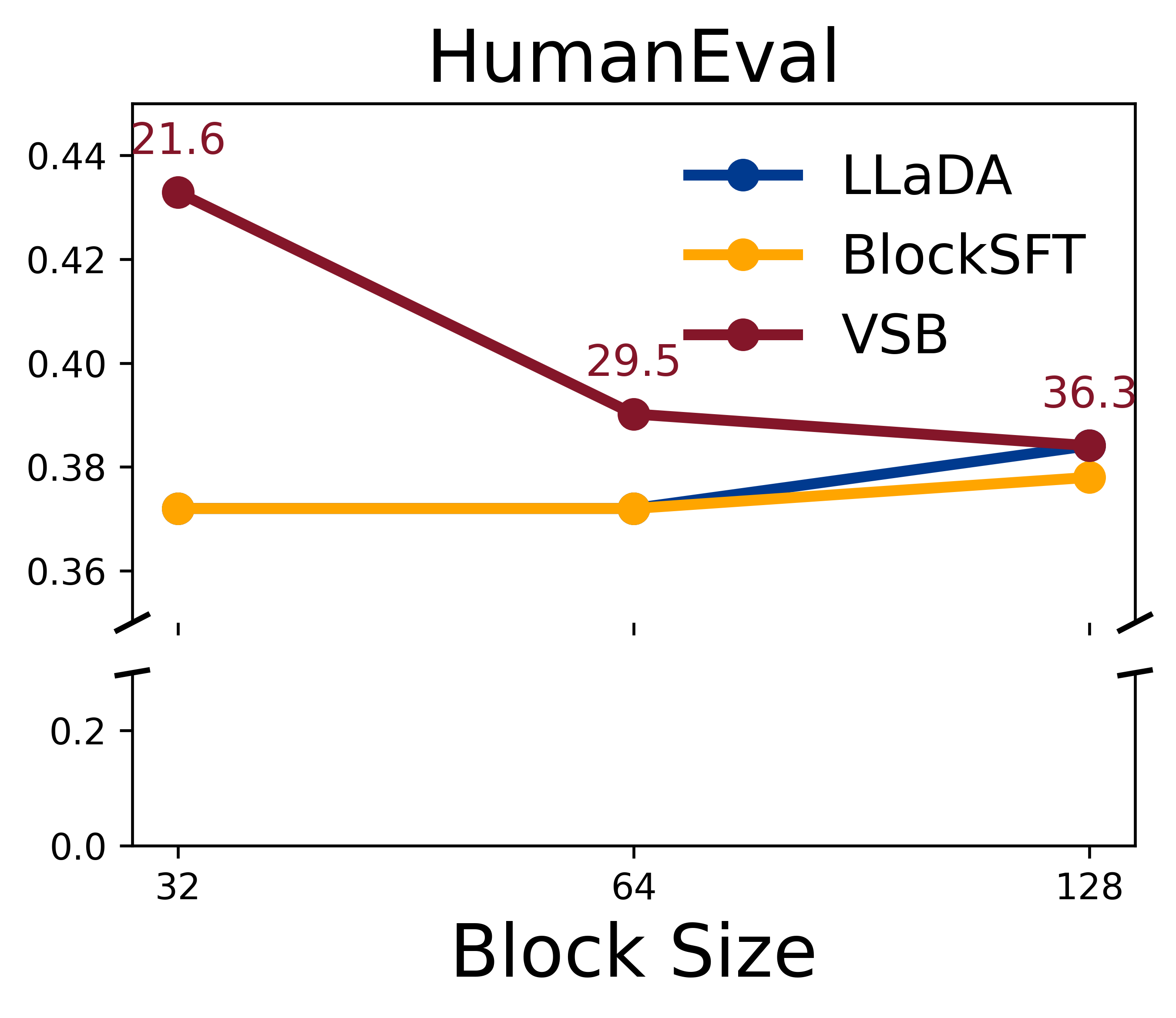}
    \vspace{-0.3cm}
    \caption{Accuracy vs. block window on GSM8K and HumanEval. The average block size of VSB is annotated in \textcolor{purple}{red}.}
    \label{fig:block_size_accuracy}
\end{figure}

\textbf{Coarse-to-fine boundary search preserves accuracy while speeding up decoding.}
Shown in Table~\ref{tab:vsb_c2f_ablation}, on MBPP, HellaSwag, and GPQA-Diamond, moderate stride settings maintain accuracy while reducing boundary evaluations and provide speedups by up to $1.6\times$. This indicates that VSB does not require an exhaustive scan over all boundaries in the budgeted block: approximate search recovers high-quality commitments without affecting self-containedness.

\begin{table}[!t]
\centering
\small
\setlength{\tabcolsep}{7pt}
\renewcommand{\arraystretch}{1.2}
\resizebox{\linewidth}{!}{
\begin{tabular}{l|cccc|cc|ccc}
\toprule
\multirow{2}{*}{\textbf{Metric}}
& \multicolumn{4}{c|}{\textbf{MBPP (64)}}
& \multicolumn{2}{c|}{\textbf{HellaSwag (10)}}
& \multicolumn{3}{c}{\textbf{GPQA-Diamond (64)}} \\

& Full
& $(8,2)$
& $(16,4)$
& $(32,8)$

& Full
& $(5,2)$

& Full
& $(8,2)$
& $(16,4)$\\
\midrule
Acc.
& 39.00 & 38.80 & 38.60 & 38.60
& 76.90 & 76.68 
& 28.79 & 28.79 & 28.79 \\

Speedup
& 1.0$\times$ & 1.5$\times$ & 1.6$\times$ & 1.6$\times$
& 1.0$\times$ & 1.2$\times$ 
& 1.0$\times$ & 1.4$\times$ & 1.5$\times$ \\
\bottomrule
\end{tabular}
}
\caption{Ablation of coarse-to-fine boundary search strides $(s_c, s_f)$ for VSB decoding.  Speedup is measured relative to the full boundary scan per dataset. Block budget is given in parenthesis.}
\vspace{-0.6cm}
\label{tab:vsb_c2f_ablation}
\end{table}
\section{Related Work}
\label{sec:related}
Our work relates to four research areas:
(1) \textbf{Discrete diffusion language models} denoise masked sequences under bidirectional attention, offering an alternative to AR generation. Representative models include LLaDA~\cite{nie2025llada,llada1_5} and Dream~\cite{dream}.
(2) \textbf{Blockwise semi-AR decoding} enables practical inference by committing tokens in semi-AR blocks rather than full-context parallel updates~\cite{Arriola2025BlockDiff,fastdllm_v2}.
(3) \textbf{Blockwise training alignment} (e.g., Blockwise-SFT) reduces the training-inference mismatch by aligning supervision with fixed-size block commitment~\cite{sun2025blockwisesft,Chen2025dParallel}.
(4) \textbf{Heuristic block selection} chooses block boundaries using fixed sizes, punctuation, or other heuristic/learned cues~\cite{Li2025DAEDAL,lu2025adablock}. VSB differs from prior heuristic boundary schemes by introducing an explicit \textit{self-containedness} criterion (Eq.~\eqref{eq:boundary_selection_rule}) that selects variable-size blocks based on predictive stability with and without future context. A detailed review is provided in Appendix~\ref{appendix:related_work}.

\section{Conclusion}
We introduced \textit{self-containedness} as a principled criterion for deciding when a block is safe to commit in blockwise decoding of discrete diffusion language models. Variable-size Self-contained Blocks (VSB) replaces fixed or heuristic boundaries with content-dependent decisions by measuring how predictions change when future context is revealed. VSB also aligns training and inference by training under the same \textit{No-Future} conditioning used at commitment time. Across reasoning, coding, and general knowledge benchmarks, VSB delivers consistent improvements over fixed-size and heuristic blockwise decoding. We hope to inspire future work on principled variable size block decoding.

\section{Impact Statement}
As a foundational study on discrete diffusion language models, we do not identify any direct negative societal impacts. Our research relies on publicly available datasets and models, all of which are properly acknowledged and pose no risks requiring safeguards. While potential societal implications may exist, none warrant specific emphasis in this study.

\section{Acknowledgment}
This research has been supported by Australian Research Council Discovery Projects (CE200100025,
DP230101196 and DE250100919).

\bibliography{vsb}

@article{sun2025blockwisesft,
  author       = {Bowen Sun and
                  Yujun Cai and
                  Ming{-}Hsuan Yang and
                  Yiwei Wang},
  title        = {Blockwise {SFT} for Diffusion Language Models: Reconciling Bidirectional
                  Attention and Autoregressive Decoding},
  journal      = {CoRR},
  year         = {2025},
  url          = {https://doi.org/10.48550/arXiv.2508.19529},
}

@article{lu2025adablock,
  author       = {Guanxi Lu and
                  Hao Mark Chen and
                  Yuto Karashima and
                  Zhican Wang and
                  Daichi Fujiki and
                  Hongxiang Fan},
  title        = {AdaBlock-dLLM: Semantic-Aware Diffusion {LLM} Inference via Adaptive
                  Block Size},
  journal      = {CoRR},
  year         = {2025},
  url          = {https://doi.org/10.48550/arXiv.2509.26432},

}

@article{blcokd2f,
  author       = {Xu Wang and
                  Chenkai Xu and
                  Yijie Jin and
                  Jiachun Jin and
                  Hao Zhang and
                  Zhijie Deng},
  title        = {Diffusion LLMs Can Do Faster-Than-AR Inference via Discrete Diffusion
                  Forcing},
  journal      = {CoRR},
  year         = {2025},
  url          = {https://doi.org/10.48550/arXiv.2508.09192},
}

@article{zhong2026beyond,
    title={Beyond Hard Masks: Progressive Token Evolution for Diffusion Language Models},
    author={Zhong, Linhao and Wu, Linyu and Fang, Bozhen and Feng, Tianjian and Jing, Chenchen and Wang, Wen and Zhang, Jiaheng and Chen, Hao and Shen, Chunhua},
    journal={arXiv preprint arXiv:2601.07351},
    year={2026}
}

@inproceedings{diffusion_lm_2022,
  author       = {Xiang Lisa Li and
                  John Thickstun and
                  Ishaan Gulrajani and
                  Percy Liang and
                  Tatsunori B. Hashimoto},
  title        = {Diffusion-LM Improves Controllable Text Generation},
  booktitle    = {NeurIPS},
  year         = {2022},
}

@inproceedings{diffusionbert,
  author       = {Zhengfu He and
                  Tianxiang Sun and
                  Qiong Tang and
                  Kuanning Wang and
                  Xuanjing Huang and
                  Xipeng Qiu},
  title        = {DiffusionBERT: Improving Generative Masked Language Models with Diffusion
                  Models},
  booktitle    = {ACL},
  year         = {2023},
}

@inproceedings{d3pm,
  author       = {Jacob Austin and
                  Daniel D. Johnson and
                  Jonathan Ho and
                  Daniel Tarlow and
                  Rianne van den Berg},
  title        = {Structured Denoising Diffusion Models in Discrete State-Spaces},
  booktitle    = {NeurIPS},
  year         = {2021},
}

@inproceedings{nie2025llada,
  author       = {Shen Nie and
                  Fengqi Zhu and
                  Zebin You and
                  Xiaolu Zhang and
                  Jingyang Ou and
                  Jun Hu and
                  Jun Zhou and
                  Yankai Lin and
                  Ji{-}Rong Wen and
                  Chongxuan Li},
  title        = {Large Language Diffusion Models},
  booktitle      = {NeurIPS},
  year         = {2025},
}

@inproceedings{Arriola2025BlockDiff,
  author       = {Marianne Arriola and
                  Aaron Gokaslan and
                  Justin T. Chiu and
                  Zhihan Yang and
                  Zhixuan Qi and
                  Jiaqi Han and
                  Subham Sekhar Sahoo and
                  Volodymyr Kuleshov},
  title        = {Block Diffusion: Interpolating Between Autoregressive and Diffusion
                  Language Models},
  booktitle    = {ICLR},
  year         = {2025},
}

@article{fastdllm,
  author       = {Chengyue Wu and
                  Hao Zhang and
                  Shuchen Xue and
                  Zhijian Liu and
                  Shizhe Diao and
                  Ligeng Zhu and
                  Ping Luo and
                  Song Han and
                  Enze Xie},
  title        = {Fast-dLLM: Training-free Acceleration of Diffusion {LLM} by Enabling
                  {KV} Cache and Parallel Decoding},
  journal      = {CoRR},
  year         = {2025},
  url          = {https://doi.org/10.48550/arXiv.2505.22618},
}

@article{Chen2025dParallel,
  author       = {Zigeng Chen and
                  Gongfan Fang and
                  Xinyin Ma and
                  Ruonan Yu and
                  Xinchao Wang},
  title        = {dParallel: Learnable Parallel Decoding for dLLMs},
  journal      = {CoRR},
  year         = {2025},
  url          = {https://doi.org/10.48550/arXiv.2509.26488},

}

@article{Xiong2025Stepwiser,
  author       = {Wei Xiong and
                  Wenting Zhao and
                  Weizhe Yuan and
                  Olga Golovneva and
                  Tong Zhang and
                  Jason Weston and
                  Sainbayar Sukhbaatar},
  title        = {StepWiser: Stepwise Generative Judges for Wiser Reasoning},
  journal      = {CoRR},
  year         = {2025},
  url          = {https://doi.org/10.48550/arXiv.2508.19229},

}

@article{Li2025DAEDAL,
  author       = {Jinsong Li and
                  Xiaoyi Dong and
                  Yuhang Zang and
                  Yuhang Cao and
                  Jiaqi Wang and
                  Dahua Lin},
  title        = {Beyond Fixed: Training-Free Variable-Length Denoising for Diffusion
                  Large Language Models},
  journal      = {CoRR},
  year         = {2025},
  url          = {https://doi.org/10.48550/arXiv.2508.00819},
}

@inproceedings{lora,
  author       = {Edward J. Hu and
                  Yelong Shen and
                  Phillip Wallis and
                  Zeyuan Allen{-}Zhu and
                  Yuanzhi Li and
                  Shean Wang and
                  Lu Wang and
                  Weizhu Chen},
  title        = {LoRA: Low-Rank Adaptation of Large Language Models},
  booktitle    = {ICLR},
  year         = {2022},
}

@article{lora_config,
  author       = {Dan Biderman and
                  Jacob P. Portes and
                  Jose Javier Gonzalez Ortiz and
                  Mansheej Paul and
                  Philip Greengard and
                  Connor Jennings and
                  Daniel King and
                  Sam Havens and
                  Vitaliy Chiley and
                  Jonathan Frankle and
                  Cody Blakeney and
                  John Patrick Cunningham},
  title        = {LoRA Learns Less and Forgets Less},
  journal      = {TMLR},
  year         = {2024},
}

@inproceedings{AdamW,
  author       = {Ilya Loshchilov and
                  Frank Hutter},
  title        = {Decoupled Weight Decay Regularization},
  booktitle    = {ICLR},
  year         = {2019},
}

@misc{dllm_repo,
    author = {Zhanhui Zhou and Lingjie Chen and Hanghang Tong and Dawn Song},
    title = {dLLM: Simple Diffusion Language Modeling},
    year = {2025},
    publisher = {GitHub},
    journal = {GitHub repository},
    howpublished = {\url{https://github.com/ZHZisZZ/dllm}},
}

@article{dllmvar,
  author       = {Yicun Yang and
                  Cong Wang and
                  Shaobo Wang and
                  Zichen Wen and
                  Biqing Qi and
                  Hanlin Xu and
                  Linfeng Zhang},
  title        = {Diffusion {LLM} with Native Variable Generation Lengths: Let {[EOS]}
                  Lead the Way},
  journal      = {CoRR},
  year         = {2025},
  url          = {https://doi.org/10.48550/arXiv.2510.24605},
}

@article{llada1_5,
  author       = {Fengqi Zhu and
                  Rongzhen Wang and
                  Shen Nie and
                  Xiaolu Zhang and
                  Chunwei Wu and
                  Jun Hu and
                  Jun Zhou and
                  Jianfei Chen and
                  Yankai Lin and
                  Ji{-}Rong Wen and
                  Chongxuan Li},
  title        = {LLaDA 1.5: Variance-Reduced Preference Optimization for Large Language
                  Diffusion Models},
  journal      = {CoRR},
  year         = {2025},
  url          = {https://doi.org/10.48550/arXiv.2505.19223},

}

@article{diffusion_survey,
  author       = {Tianyi Li and
                  Mingda Chen and
                  Bowei Guo and
                  Zhiqiang Shen},
  title        = {A Survey on Diffusion Language Models},
  journal      = {CoRR},
  year         = {2025},
  url          = {https://doi.org/10.48550/arXiv.2508.10875},
}

@article{diffusion_survey2,
  author       = {Runpeng Yu and
                  Qi Li and
                  Xinchao Wang},
  title        = {Discrete Diffusion in Large Language and Multimodal Models: {A} Survey},
  journal      = {CoRR},
  year         = {2025},
  url          = {https://doi.org/10.48550/arXiv.2506.13759},
}

@article{didi,
  author       = {Haoyang Zheng and
                  Xinyang Liu and
                  Cindy Xiangrui Kong and
                  Nan Jiang and
                  Zheyuan Hu and
                  Weijian Luo and
                  Wei Deng and
                  Guang Lin},
  title        = {Ultra-Fast Language Generation via Discrete Diffusion Divergence Instruct},
  journal      = {CoRR},
  year         = {2025},
  url          = {https://doi.org/10.48550/arXiv.2509.25035},
}

@misc{llada2.0,
      title={LLaDA2.0: Scaling Up Diffusion Language Models to 100B}, 
      author={Tiwei Bie and Maosong Cao and Kun Chen and Lun Du and Mingliang Gong and Zhuochen Gong and Yanmei Gu and Jiaqi Hu and Zenan Huang and Zhenzhong Lan and Chengxi Li and Chongxuan Li and Jianguo Li and Zehuan Li and Huabin Liu and Ling Liu and Guoshan Lu and Xiaocheng Lu and Yuxin Ma and Jianfeng Tan and Lanning Wei and Ji-Rong Wen and Yipeng Xing and Xiaolu Zhang and Junbo Zhao and Da Zheng and Jun Zhou and Junlin Zhou and Zhanchao Zhou and Liwang Zhu and Yihong Zhuang},
      year={2025},
      eprint={2512.15745},
      archivePrefix={arXiv},
      primaryClass={cs.LG},
      url={https://arxiv.org/abs/2512.15745}, 
}

@inproceedings{blockdecode_for_ar,
  author       = {Mitchell Stern and
                  Noam Shazeer and
                  Jakob Uszkoreit},
  editor       = {Samy Bengio and
                  Hanna M. Wallach and
                  Hugo Larochelle and
                  Kristen Grauman and
                  Nicol{\`{o}} Cesa{-}Bianchi and
                  Roman Garnett},
  title        = {Blockwise Parallel Decoding for Deep Autoregressive Models},
  booktitle    = {NeurIPS},
  year         = {2018},
}

@inproceedings{sat,
  author       = {Chunqi Wang and
                  Ji Zhang and
                  Haiqing Chen},
  editor       = {Ellen Riloff and
                  David Chiang and
                  Julia Hockenmaier and
                  Jun'ichi Tsujii},
  title        = {Semi-Autoregressive Neural Machine Translation},
  booktitle    = {EMNLP},
  year         = {2018},
}

@misc{efficeint_dlm,
      title={Efficient-DLM: From Autoregressive to Diffusion Language Models, and Beyond in Speed}, 
      author={Yonggan Fu and Lexington Whalen and Zhifan Ye and Xin Dong and Shizhe Diao and Jingyu Liu and Chengyue Wu and Hao Zhang and Enze Xie and Song Han and Maksim Khadkevich and Jan Kautz and Yingyan Celine Lin and Pavlo Molchanov},
      year={2025},
      eprint={2512.14067},
      archivePrefix={arXiv},
      primaryClass={cs.CL},
      url={https://arxiv.org/abs/2512.14067}, 
}

@article{anyorder_diff,
  author       = {Jaeyeon Kim and
                  Cheuk Kit Lee and
                  Carles Domingo{-}Enrich and
                  Yilun Du and
                  Sham M. Kakade and
                  Timothy Ngotiaoco and
                  Sitan Chen and
                  Michael S. Albergo},
  title        = {Any-Order Flexible Length Masked Diffusion},
  journal      = {CoRR},
  year         = {2025},
  url          = {https://doi.org/10.48550/arXiv.2509.01025},
}

@article{dream,
  author       = {Jiacheng Ye and
                  Zhihui Xie and
                  Lin Zheng and
                  Jiahui Gao and
                  Zirui Wu and
                  Xin Jiang and
                  Zhenguo Li and
                  Lingpeng Kong},
  title        = {Dream 7B: Diffusion Large Language Models},
  journal      = {CoRR},
  year         = {2025},
  url          = {https://doi.org/10.48550/arXiv.2508.15487},
}

@article{liu2025wedlm,
  title={WeDLM: Reconciling Diffusion Language Models with Standard Causal Attention for Fast Inference},
  author={Liu, Aiwei and He, Minghua and Zeng, Shaoxun and Zhang, Linhao and Wu, Chuhan and Jia, Wei and Liu, Yuan and Yu, Yang and Zhou, Xiao and Zhou, Jie},
  journal={arXiv preprint arXiv:2512.22737},
  year={2025}
}

@inproceedings{MDLM,
  author       = {Subham S. Sahoo and
                  Marianne Arriola and
                  Yair Schiff and
                  Aaron Gokaslan and
                  Edgar Marroquin and
                  Justin T. Chiu and
                  Alexander Rush and
                  Volodymyr Kuleshov},
  editor       = {Amir Globersons and
                  Lester Mackey and
                  Danielle Belgrave and
                  Angela Fan and
                  Ulrich Paquet and
                  Jakub M. Tomczak and
                  Cheng Zhang},
  title        = {Simple and Effective Masked Diffusion Language Models},
  booktitle    = {NeurIPS},
  year         = {2024},
}

@article{ni2025dllm_vs_ar,
  author       = {Jinjie Ni and
                  Qian Liu and
                  Longxu Dou and
                  Chao Du and
                  Zili Wang and
                  Hang Yan and
                  Tianyu Pang and
                  Michael Qizhe Shieh},
  title        = {Diffusion Language Models are Super Data Learners},
  journal      = {CoRR},
  year         = {2025},
  url          = {https://doi.org/10.48550/arXiv.2511.03276},
}

@article{soft_masked_dllm,
  author       = {Michael Hersche and
                  Samuel Moor{-}Smith and
                  Thomas Hofmann and
                  Abbas Rahimi},
  title        = {Soft-Masked Diffusion Language Models},
  journal      = {CoRR},
  year         = {2025},
  url          = {https://doi.org/10.48550/arXiv.2510.17206},
}

@article{fastdllm_v2,
  author       = {Chengyue Wu and
                  Hao Zhang and
                  Shuchen Xue and
                  Shizhe Diao and
                  Yonggan Fu and
                  Zhijian Liu and
                  Pavlo O. Molchanov and
                  Ping Luo and
                  Song Han and
                  Enze Xie},
  title        = {Fast-dLLM v2: Efficient Block-Diffusion {LLM}},
  journal      = {CoRR},
  year         = {2025},
  url          = {https://doi.org/10.48550/arXiv.2509.26328},
}

@article{dllmcache,
  author       = {Zhiyuan Liu and
                  Yicun Yang and
                  Yaojie Zhang and
                  Junjie Chen and
                  Chang Zou and
                  Qingyuan Wei and
                  Shaobo Wang and
                  Linfeng Zhang},
  title        = {dLLM-Cache: Accelerating Diffusion Large Language Models with Adaptive
                  Caching},
  journal      = {CoRR},
  year         = {2025},
  url          = {https://doi.org/10.48550/arXiv.2506.06295},
}

@article{seeddiffusion,
  author       = {Yuxuan Song and
                  Zheng Zhang and
                  Cheng Luo and
                  Pengyang Gao and
                  Fan Xia and
                  Hao Luo and
                  Zheng Li and
                  Yuehang Yang and
                  Hongli Yu and
                  Xingwei Qu and
                  Yuwei Fu and
                  Jing Su and
                  Ge Zhang and
                  Wenhao Huang and
                  Mingxuan Wang and
                  Lin Yan and
                  Xiaoying Jia and
                  Jingjing Liu and
                  Wei{-}Ying Ma and
                  Ya{-}Qin Zhang and
                  Yonghui Wu and
                  Hao Zhou},
  title        = {Seed Diffusion: {A} Large-Scale Diffusion Language Model with High-Speed
                  Inference},
  journal      = {CoRR},
  year         = {2025},
  url          = {https://doi.org/10.48550/arXiv.2508.02193},
}

@article{dimple,
  author       = {Runpeng Yu and
                  Xinyin Ma and
                  Xinchao Wang},
  title        = {Dimple: Discrete Diffusion Multimodal Large Language Model with Parallel
                  Decoding},
  journal      = {CoRR},
  year         = {2025},
  url          = {https://doi.org/10.48550/arXiv.2505.16990},
}

@article{gsm8k,
  author       = {Karl Cobbe and
                  Vineet Kosaraju and
                  Mohammad Bavarian and
                  Mark Chen and
                  Heewoo Jun and
                  Lukasz Kaiser and
                  Matthias Plappert and
                  Jerry Tworek and
                  Jacob Hilton and
                  Reiichiro Nakano and
                  Christopher Hesse and
                  John Schulman},
  title        = {Training Verifiers to Solve Math Word Problems},
  journal      = {CoRR},
  year         = {2021},
  url          = {https://arxiv.org/abs/2110.14168},
}

@inproceedings{original_math500,
  author       = {Dan Hendrycks and
                  Collin Burns and
                  Saurav Kadavath and
                  Akul Arora and
                  Steven Basart and
                  Eric Tang and
                  Dawn Song and
                  Jacob Steinhardt},
  title        = {Measuring Mathematical Problem Solving With the {MATH} Dataset},
  booktitle    = {NeurIPS Datasets and Benchmarks},
  year         = {2021},
}

@inproceedings{math500,
  author       = {Hunter Lightman and
                  Vineet Kosaraju and
                  Yuri Burda and
                  Harrison Edwards and
                  Bowen Baker and
                  Teddy Lee and
                  Jan Leike and
                  John Schulman and
                  Ilya Sutskever and
                  Karl Cobbe},
  title        = {Let's Verify Step by Step},
  booktitle    = {ICLR},
  year         = {2024},
}

@article{humaneval,
  author       = {Mark Chen and
                  Jerry Tworek and
                  Heewoo Jun and
                  Qiming Yuan and
                  Henrique Pond{\'{e}} de Oliveira Pinto and
                  Jared Kaplan and
                  Harri Edwards and
                  Yuri Burda and
                  Nicholas Joseph and
                  Greg Brockman and
                  Alex Ray and
                  Raul Puri and
                  Gretchen Krueger and
                  Michael Petrov and
                  Heidy Khlaaf and
                  Girish Sastry and
                  Pamela Mishkin and
                  Brooke Chan and
                  Scott Gray and
                  Nick Ryder and
                  Mikhail Pavlov and
                  Alethea Power and
                  Lukasz Kaiser and
                  Mohammad Bavarian and
                  Clemens Winter and
                  Philippe Tillet and
                  Felipe Petroski Such and
                  Dave Cummings and
                  Matthias Plappert and
                  Fotios Chantzis and
                  Elizabeth Barnes and
                  Ariel Herbert{-}Voss and
                  William Hebgen Guss and
                  Alex Nichol and
                  Alex Paino and
                  Nikolas Tezak and
                  Jie Tang and
                  Igor Babuschkin and
                  Suchir Balaji and
                  Shantanu Jain and
                  William Saunders and
                  Christopher Hesse and
                  Andrew N. Carr and
                  Jan Leike and
                  Joshua Achiam and
                  Vedant Misra and
                  Evan Morikawa and
                  Alec Radford and
                  Matthew Knight and
                  Miles Brundage and
                  Mira Murati and
                  Katie Mayer and
                  Peter Welinder and
                  Bob McGrew and
                  Dario Amodei and
                  Sam McCandlish and
                  Ilya Sutskever and
                  Wojciech Zaremba},
  title        = {Evaluating Large Language Models Trained on Code},
  journal      = {CoRR},
  year         = {2021},
  url          = {https://arxiv.org/abs/2107.03374},

}

@article{mbpp,
  author       = {Jacob Austin and
                  Augustus Odena and
                  Maxwell I. Nye and
                  Maarten Bosma and
                  Henryk Michalewski and
                  David Dohan and
                  Ellen Jiang and
                  Carrie J. Cai and
                  Michael Terry and
                  Quoc V. Le and
                  Charles Sutton},
  title        = {Program Synthesis with Large Language Models},
  journal      = {CoRR},
  year         = {2021},
  eprinttype    = {arXiv},
  url          = {https://arxiv.org/abs/2108.07732},

}

@article{mmlu,
  author       = {Dan Hendrycks and
                  Collin Burns and
                  Steven Basart and
                  Andy Zou and
                  Mantas Mazeika and
                  Dawn Song and
                  Jacob Steinhardt},
  title        = {Measuring Massive Multitask Language Understanding},
  journal      = {ICLR},
  year         = {2021},
}

@article{SDAR,
  author       = {Shuang Cheng and
                  Yihan Bian and
                  Dawei Liu and
                  Linfeng Zhang and
                  Qian Yao and
                  Zhongbo Tian and
                  Wenhai Wang and
                  Qipeng Guo and
                  Kai Chen and
                  Biqing Qi and
                  Bowen Zhou},
  title        = {{SDAR:} {A} Synergistic Diffusion-AutoRegression Paradigm for Scalable
                  Sequence Generation},
  journal      = {CoRR},
  year         = {2025},
  url          = {https://doi.org/10.48550/arXiv.2510.06303},
}

@misc{NB-DIFF,
      title={From Next-Token to Next-Block: A Principled Adaptation Path for Diffusion LLMs}, 
      author={Yuchuan Tian and Yuchen Liang and Jiacheng Sun and Shuo Zhang and Guangwen Yang and Yingte Shu and Sibo Fang and Tianyu Guo and Kai Han and Chao Xu and Hanting Chen and Xinghao Chen and Yunhe Wang},
      year={2025},
      eprint={2512.06776},
      archivePrefix={arXiv},
      url={https://arxiv.org/abs/2512.06776}, 
}

@inproceedings{mmlu_ethics,
  author       = {Dan Hendrycks and
                  Collin Burns and
                  Steven Basart and
                  Andrew Critch and
                  Jerry Li and
                  Dawn Song and
                  Jacob Steinhardt},
  title        = {Aligning {AI} With Shared Human Values},
  booktitle    = {ICLR},
  year         = {2021},
}

@inproceedings{hellaswag,
  author       = {Rowan Zellers and
                  Ari Holtzman and
                  Yonatan Bisk and
                  Ali Farhadi and
                  Yejin Choi},
  editor       = {Anna Korhonen and
                  David R. Traum and
                  Llu{\'{\i}}s M{\`{a}}rquez},
  title        = {HellaSwag: Can a Machine Really Finish Your Sentence?},
  booktitle    = {ACL},
  year         = {2019},
}

@article{GPQA-Diamond,
  author       = {David Rein and
                  Betty Li Hou and
                  Asa Cooper Stickland and
                  Jackson Petty and
                  Richard Yuanzhe Pang and
                  Julien Dirani and
                  Julian Michael and
                  Samuel R. Bowman},
  title        = {{GPQA:} {A} Graduate-Level Google-Proof Q{\&}A Benchmark},
  journal      = {CoRR},
  year         = {2023},
  eprinttype    = {arXiv},
  url          = {https://doi.org/10.48550/arXiv.2311.12022},

}
\bibliographystyle{icml2026}

\newpage
\appendix
\onecolumn
\section*{Appendix}

\section{Potential Limitations} \label{appendix:limitations}
The current study focuses on discrete diffusion language models with \textit{native} global bidirectional attention, where the \textit{No-Future} and \textit{Future-Aware} conditionals can be queried directly within a unified architecture. \textbf{Extending VSB to hybrid designs that combine an AR backbone with diffusion-style refinement (e.g., AR-initialised diffusion models) is not explored in this work.} Such models introduce additional interfaces between the AR proposal and bidirectional denoising that may alter the interpretation of self-containedness and the design of NF/FA probes. Adapting VSB to these settings (e.g., Dream~\cite{dream}, LLaDA 2.0~\cite{llada2.0}, FastdLLM v2~\cite{fastdllm_v2}, SDAR~\cite{sdar}, NBDiff~\cite{NB-DIFF} WeDLM~\cite{liu2025wedlm}-style pipelines) is left for future work. As described in main text, our formulation provides a clean reference implementation of self-contained decoding without complex engineering optimisations. \textbf{While introducing a novel self-contained variable-size block method with great performance, it indeed comes with additional computational cost.} In particular, VSB (Algorithm~\ref{alg:training}) introduces an explicit boundary selection step: for each block, candidate boundaries within a lookahead window are scored via NF/FA divergence, and the best boundary is selected by maximising the boundary selection objective Eq.~\eqref{eq:boundary_selection_rule}. When implemented as an exhaustive scan, this procedure evaluates all candidate boundaries within the window, which becomes costly when the block length is large. Nonetheless, \textbf{we have made the effort to provide a coarse-to-fine boundary search strategy} that first evaluates a sparse subset of boundaries and then refines the search locally around the most promising candidates, reducing the number of NF evaluations in practice while preserving most of the adaptivity benefits. The corresponding computational complexity improvement is provided in Section~\ref{sec:coarse-fine} 

Moreover, although prefix KV caching can improve efficiency in blockwise decoding, its implementation is less straightforward and partially applicable in the current VSB formulation. VSB explicitly compares \textit{No-Future} and \textit{Future-Aware} conditionals at multiple candidate boundaries, which requires evaluating the model under varying context truncations and can therefore limit cache reuse. In addition, under globally bidirectional attention, standard KV caching may alter the effective conditioning by freezing intermediate representations of previously processed tokens, potentially compromising the fidelity of NF/FA comparisons. Nevertheless, \textbf{we have made our best effort to adopt a future-aware activation caching strategy }that reuses intermediate hidden states from a full-context forward and recomputes only a small number of upper layers under masked attention. This approach preserves the semantics of self-containedness while providing meaningful speedups for boundary selection. Extending KV caching to support future-aware self-containedness tests is left as an engineering problem for future work.

\section{Proof of Technical Insights}

\subsection{Proof of Theorem~\ref{thm:consistency}} \label{proof:self_contained_commit}
\begin{proof}
For each position \(i \in (p,b^\star]\), let \(P_i^{\mathrm{NF}}\) and
\(P_i^{\mathrm{FA}}\) denote the predictive distributions under No-Future (NF)
and Future-Aware (FA) conditioning. We seek to bound their average discrepancy
within the block.

Consider the choice $\mathcal{D}(P,Q)=D_{\mathrm{KL}}(P\|Q)$, by Pinsker’s inequality:
$$
\delta_{\mathrm{TV}}\!\left(P_i^{\mathrm{NF}}, P_i^{\mathrm{FA}}\right)
\le
\sqrt{\tfrac{1}{2}\,
D_{\mathrm{KL}}\!\left(P_i^{\mathrm{NF}}\|P_i^{\mathrm{FA}}\right)}.
$$
Given that the square-root $\sqrt{.}$ function is concave, applying the Jensen’s inequality and  averaging over positions \(i=p+1,\ldots,b^\star\), we have:
$$
\frac{1}{b^\star-p}\sum_{i=p+1}^{b^\star}
\delta_{\mathrm{TV}}\!\left(P_i^{\mathrm{NF}}, P_i^{\mathrm{FA}}\right)
\le
\sqrt{
\frac{1}{2(b^\star-p)}
\sum_{i=p+1}^{b^\star}
D_{\mathrm{KL}}\!\left(P_i^{\mathrm{NF}}\|P_i^{\mathrm{FA}}\right)
}.
$$
By definition of \(\varepsilon\)-self-containedness:
$$
\frac{1}{b^\star-p}
\sum_{i=p+1}^{b^\star}
D_{\mathrm{KL}}\!\left(P_i^{\mathrm{NF}}\|P_i^{\mathrm{FA}}\right)
\le
\varepsilon.
$$
Substituting this bound gives:
$$
\frac{1}{b^\star-p}\sum_{i=p+1}^{b^\star}
\delta_{\mathrm{TV}}\!\left(P_i^{\mathrm{NF}}, P_i^{\mathrm{FA}}\right)
\le
\sqrt{\tfrac{\varepsilon}{2}}
=
\mathcal{O}(\sqrt{\varepsilon}).
$$
Thus, when a block is \(\varepsilon\)-self-contained, revealing future context would only change the model's token-level predictive distributions within that block by at most $\mathcal{O}(\sqrt{\varepsilon})$ on average. This formalises the notion that self-contained blocks are approximately future-independent.
\end{proof}

\subsection{Proof of Proposition~\ref{prop:fixed_length_characterisation}}
\label{proof:fixed_length_characterisation}
\begin{proof} Given a fixed-size block commitment with length $W>0$.
($\Rightarrow$) If all committed blocks satisfy $S(p,b)\le\varepsilon$, then in
particular each committed block $(p,p+W]$ satisfies $S(p,p+W)\le\varepsilon$.

($\Leftarrow$) If $S(p,p+W)\le\varepsilon$ holds for every committed block
$(p,p+W]$, then the strategy satisfies the condition for all committed blocks.
\end{proof}

\section{Extended Related Work} \label{appendix:related_work}
\paragraph{Discrete diffusion language models.}
Diffusion-based language modeling presents a promising alternative to the left-to-right token generation with an iterative denoising process that predicts masked tokens under bidirectional attention~\cite{diffusion_survey,nie2025llada}. This paradigm enables parallel token updates, arbitrary-order refinement, and natural support for infilling. Early work such as Diffusion-LM studied continuous or discrete diffusion for controllable text generation \cite{diffusion_lm_2022,d3pm,diffusionbert,MDLM}, motivating subsequent scaling of diffusion objectives to large Transformer backbones. Recent discrete dLLMs, notably LLaDA \cite{nie2025llada} and Dream \cite{dream}, demonstrate that masked discrete diffusion training can reach competitive capability at the multi-billion scale compare to auto-regressive models~\cite{liu2025wedlm,llada2.0}. A key practical issue is efficiency and length control. Standard diffusion decoding typically requires many denoising iterations and often assumes a fixed generation length. Several lines of work address these constraints~\cite{diffusion_survey2,seeddiffusion,dllmcache,zhong2026beyond}. DAEDAL proposes a training-free strategy that expands length by inserting mask tokens when the current allocation appears insufficient, reducing wasted computation from overly long allocations while avoiding truncation from short ones~\cite{Li2025DAEDAL}. dLLM-Var instead modifies training to predict \texttt{[EOS]} so that variable-length generation is native to the model, while still supporting blockwise diffusion-style inference~\cite{dllmvar}. On acceleration, D2F converts a pretrained dLLM into an AR-diffusion hybrid through a teacher-student distillation process that distill the teacher full bidirectional attention predictions into a student model with blockwise causal attention.~\cite{blcokd2f}, while DiDi-Instruct distills a pretrained masked dLLM into a faster student with substantially fewer denoising steps~\cite{didi}. Our focus is orthogonal to these methods by target the \textit{training-inference mismatch} created by combining bidirectional diffusion training with blockwise semi-AR commitment, and this is achieved by making block boundaries \textit{variable} under our proposed \textit{self-containedness} principle.

\paragraph{Blockwise semi-autoregressive decoding.}
In practice, many dLLMs decode using a blockwise semi-autoregressive procedure: tokens are denoised in parallel within a window, then a prefix block is committed and the window shifts forward~\cite{Arriola2025BlockDiff,efficeint_dlm,fastdllm_v2}. This pattern benefits from KV caching and provides a favorable trade-off between latency and quality, but it introduces an additional design choice: where to place boundaries~\cite{fastdllm,lu2025adablock,dllmcache}. Historically, blockwise and semi-autoregressive decoding has been studied for autoregressive models and sequence-to-sequence generation like Blockwise parallel decoding and SAT~\cite{blockdecode_for_ar,sat}. For dLLMs, the boundary choice interacts strongly with bidirectional attention: tokens inside a block can exploit within-block context during denoising, but committed tokens become immutable and must remain consistent with later context. Fixed-size blocks, punctuation heuristics, or other surface rules can therefore commit tokens that would have been revised once future context becomes visible, propagating inconsistencies across future blocks~\cite{Chen2025dParallel,sun2025blockwisesft}. Our method formalises this risk through an explicit \textit{self-containedness} commitment criterion (Eq.~\eqref{eq:self_contain_score}), and selects variable-size blocks by trading off commit length against self-containedness. This differs from the fixed-size blocks or heuristic boundaries approach, where our the goal is to ensure that committed content are self-contained under the semantics available at commitment time without needing future available context to be refined.

\paragraph{Blockwise training alignment.}
A central challenge is that conventional diffusion training and SFT mask tokens across the whole input response, whereas blockwise semi-AR inference commits blocks without seeing the future~\cite{Arriola2025BlockDiff}. This discrepancy can yield a mismatch in conditioning semantics: training gradients may encourage token predictions that implicitly depend on future tokens that will not be available when the block is committed~\cite{fastdllm_v2,sun2025blockwisesft,Chen2025dParallel,zhong2026beyond}. Blockwise SFT targets this mismatch by aligning fine-tuning with fixed-size blockwise decoding through training on a selected active block while masking all future tokens~\cite{sun2025blockwisesft}. This directly matches the inference-time information pattern for a chosen block size. While VSB shares the objective of aligning training to blockwise inference, it differs in how the blocks were principally chosen. Rather than assuming a fixed block length, we introduce an explicit self-containedness criterion that guides \textit{variable-size} blocks. Concretely, we evaluate candidate boundaries by comparing predictions for the current decoding region under \textit{future-aware} versus \textit{no-future} conditioning. Training then applies supervision under the same no-future truncation semantics implied by the selected boundary, reducing mismatch not just for a single fixed block size but for a \textit{content-dependent} boundary policy. In this sense, our method can be viewed as complementing Blockwise SFT, where we aim for an effective boundary selection driven by \textit{self-containedness}.

\paragraph{Semantic-aware block generation.}
Recent work has challenged the fixed block size assumption and proposed semantic-aware methods to guide generation~\cite{Li2025DAEDAL,dllmvar,anyorder_diff,Xiong2025Stepwiser}. AdaBlock-dLLM presents a training-free approach that uses confidence dynamics of semantic-associated tokens like punctuations during denoising to adjust block size and better align boundaries with local semantic structure \cite{lu2025adablock}. Related ideas also appear in multimodal discrete diffusion models, where Dimple provides a ``confident decoding'' strategy that unmasks tokens when confidence exceeds a threshold, reducing the number of iterations needed in practice \cite{dimple}. While these methods adapt decoding based on confidence heuristics, they typically do not directly enforce that a committed block would remain resolved once future context is revealed. Our method is aligned with the broader motivation of semantic coherence, but provides a different mechanism. Instead of defining steps by prompts, punctuation, or confidence-only rules, we test self-containedness via predictive consistency with and without future context. This yields a principled boundary signal that can be used both at inference time for variable-size commitment and at training time for alignment under the same no-future semantics.

\section{Extended Dataset Description} \label{appendix:Extended_dataset_description}

\paragraph{\textsc{GSM8K}~\cite{gsm8k}} GSM8K is a benchmark of grade-school level math word problems created to test multi-step reasoning and arithmetic skills in language models. It contains a total of approximately 8,500 problems with human-written solutions requiring elementary calculation steps to reach a numerical answer. The test split consists of 1,319 test examples. We evaluate under the 0-shot scenario.

\paragraph{\textsc{MATH500}~\cite{math500}} MATH-500 is a curated subset of 500 competition-style math problems drawn from the broader MATH benchmark~\cite{original_math500} that was used in OpenAI's work~\cite{math500}. Each problem is intended to challenge models' mathematical reasoning by requiring multi-step solution strategies that go beyond simple arithmetic. We evaluate under the 4-shot scenario.

\paragraph{\textsc{MMLU Generative}~\cite{mmlu,mmlu_ethics}} MMLU generative is the generative adaptation of the Massive Multitask Language Understanding (MMLU) benchmark, which spans a wide range of academic subjects across STEM, humanities, and social sciences. The underlying benchmark consists of multiple-choice questions across 57 subjects. The generative is a generation variant where the model is asked to generate the answer letter. We evaluate under the 0-shot scenario.

\paragraph{\textsc{HumanEval}~\cite{humaneval}} HumanEval is a code generation benchmark designed to evaluate Python function synthesis from docstrings. It contains 164 programming problems, each with a natural language ``docstring" specification and accompanying unit tests. The test set consists of 164 questions, with correctness determined by executing generated code against the groundtruth. We evaluate under the 0-shot scenario.

\paragraph{\textsc{MBPP}~\cite{mbpp}} MBPP (Mostly Basic Programming Problems) is a coding dataset of Python programming questions aimed at assessing fundamental coding abilities. The testset includes 500 short programming tasks. We use the instruction following variant of this dataset with 3-shot for evaluation.

\paragraph{\textsc{GPQA-Diamond}~\cite{GPQA-Diamond}} GPQA-Diamond is a difficult subset of the Graduate-Level Google-Proof Q\&A (GPQA) benchmark that focuses on expert-validated, high-difficulty multiple-choice science questions. The Diamond subset contains 198 questions, representing the hardest portion of the larger GPQA set. We evaluate under the 0-shot scenario.

\paragraph{\textsc{HellaSwag}~\cite{hellaswag}} HellaSwag is a commonsense reasoning benchmark that asks models to select the most suitable followup given a narrative context. The dataset is structured for multiple-choice evaluation and contains 10,042 test examples. We evaluate under the 0-shot scenario.

\section{Implementation Details}\label{appendix:implementation_details}
We use the {\fontfamily{qcr}\selectfont lm-evaluation-harness} framework (v0.4.10.dev0 )~\footnote{\url{https://github.com/EleutherAI/lm-evaluation-harness/tree/main}} for an unified evaluation of publicly available datasets. Experiments were conducted using Python 3.12.11, PyTorch 2.7.0+gitf717b2a (ROCm compatible) and ROCm v6.4.3 with AMD Instinct MI300X GPUs. LLaDA-8B-Instruct~\cite{nie2025llada} is used as the primary pretrained backbone for all training or post-hoc baselines. For VSB, we additionally report results on both LLaDA-1.5~\cite{llada1_5}. Dataset-specific token budgets, block sizes, and few-shot settings follow standard settings adopted in prior work~\cite{dllm_repo, nie2025llada,lu2025adablock,Arriola2025BlockDiff} and are reported in Table~\ref{tab:main_results_transposed}. During training, we sweep the block length $W \in \{32, 64, 96\}$ with default $W = 96$. To ensure fair comparison, the block length $W$ at inference time is matched to the block size used by the baselines. The coarse and fine strides are set to $s_c = \min($token budget, $16)$, and $s_f = 4$, respectively. We use the medium step to construct the hypothesis (e.g., $t_\text{hyp} = \lfloor{T / 2}\rfloor$). For models that require supervised fine-tuning, we use the training split of the {\fontfamily{qcr}\selectfont openai/gsm8k} dataset. We apply LoRA fine-tuning for training based methods under standard configuration presented in prior works~\cite{sun2025blockwisesft,lora,lora_config}. VSB is trained with the AdamW optimiser~\cite{AdamW} and cosine learning-rate schedule, using bf16 precision. Inconsistencies of results with the original reported statistics may stem from dataset configuration, software environment differences, and unavailable hyperparameters. We have reproduced the results to the best of our ability while maintaining a consistent evaluation configuration across the methods.

\section{Algorithms} \label{appendix:algorithms}
To facilitate understanding of VSB, several algorithms are presented. Algorithm~\ref{alg:training} shows the \textit{No-Future} alignment training process of VSB, which aligns the global bidirectional attention training with block-wise semi-AR inference conditioning. Algorithm~\ref{alg:vsb_inference} demonstrates the inference of VSB (non-accelerated full boundary scan version), while Algorithm~\ref{alg:vsb_coarse_to_fine} illustrates the boundary selection optimised coarse-to-fine search.

\begin{figure*}[t]
\centering

\begin{minipage}[t]{0.48\textwidth}
\begin{algorithm}[H]
\caption{VSB Training with \textit{No-Future} Alignment}
\label{alg:training}
\begin{algorithmic}[1]
\STATE {\bfseries Input:} Input sequences $x$, dLLM $p_\theta$, diffusion steps $t \in \{1,\dots,T\}$, block length $W$
\WHILE{not converged}
    \STATE Sample prefix end $p \sim \mathrm{Unif}\{1,\dots,L-1\}$
    \STATE Set window end $w \leftarrow \min(p+W, L)$    
    \STATE Define candidate boundaries $\mathcal{B} = \{b \mid p < b \le w\}$

    \STATE \textbf{(A) Construct scoring hypotheses}
    \STATE Set $\hat{x}_{1:p} \leftarrow x_{1:p}$
    \STATE Initialise $\hat{x}_{p+1:w} \leftarrow \text{Masked}$
    \STATE Construct window hypotheses:
    $
    \hat{x}_{p+1:w} \leftarrow
    \text{Eq.~\eqref{eq:tr_hypothesis}}
    $

    \STATE \textbf{(B) Compute \textit{Future-Aware} prediction}
    \STATE $f^{\mathrm{FA}}_i \leftarrow p_\theta(\cdot \mid \hat{x}_{1:w})$ for all $i \in (p, w]$

    \STATE \textbf{(C) Boundary scoring}
    \FOR{each $b \in \mathcal{B}$}
        \STATE $f^{\mathrm{NF}}_i(b) \leftarrow p_\theta(\cdot \mid \hat{x}_{1:b})$ for $i \in (p, b]$
        \STATE $S(p,b) \leftarrow$ Eq.~\eqref{eq:self_contain_score}
    \ENDFOR
    \STATE Select boundary $b^{\star} \leftarrow$ Eq.~\eqref{eq:boundary_selection_rule}

    \STATE \textbf{(D) No-Future diffusion training}
    \STATE Corrupt tokens in $(p, b^{\star}]$ to obtain $\tilde{x}^{t}_{1:b^{\star}}$
    \STATE Compute \textit{No-Future} loss:
    $
    \mathcal{L}_{\mathrm{NF}} \leftarrow
    \text{Eq.~\eqref{eq:final_nf_diff_loss}}
    $
    \STATE Update parameters:
    $
    \theta \leftarrow \theta - \eta \nabla_\theta \mathcal{L}_{\mathrm{NF}}
    $
\ENDWHILE
\end{algorithmic}
\end{algorithm}
\end{minipage}
\hfill
\begin{minipage}[t]{0.48\textwidth}
\begin{algorithm}[H]
\caption{VSB Inference}
\label{alg:vsb_inference}
\begin{algorithmic}[1]
\STATE {\bfseries Input:} Prompt $x_{1:p}$, dLLM $p_\theta$, denoising steps $t \in \{T,\dots,1\}$, block length $W$, hypothesis step $t_{\text{hyp}}$
\STATE Initialise $p \leftarrow \ell_{\text{prompt}}$
\WHILE{$p < L$}
    \STATE Set $w \leftarrow \min(p + W, L)$
    \STATE Initialise $(p,w]$ as masked
    \STATE Run diffusion from timestep $t=T$ to $t=1$ to obtain $\hat{x}_{1:w}$ and cache hypothesis $\hat{x}^{t_{\text{hyp}}}_{1:w}$ at step $t_{\text{hyp}}$
    \STATE Compute FA logits on $(p,w]$ using $\hat{x}^{t_{\text{hyp}}}_{1:w}$
    \STATE
    \STATE \# \textit{Boundary Selection}
    \FOR{each $b \in \mathcal{B}(p,w)$}
        \STATE Compute NF logits by truncation to $b$
        \STATE Compute $S(p,b) \leftarrow$ Eq.~\eqref{eq:self_contain_score}
    \ENDFOR
    \STATE Select $b^\star$ by Eq.~\eqref{eq:boundary_selection_rule}
    \STATE Commit $\hat{x}_{p+1:b^\star}$
    \STATE Update $p \leftarrow b^\star$
\ENDWHILE
\STATE \textbf{return} generated sequence
\end{algorithmic}
\end{algorithm}
\end{minipage}

\end{figure*}

\begin{algorithm}[H]
\caption{VSB Boundary Selection via Coarse-to-Fine Search}
\label{alg:vsb_coarse_to_fine}
\begin{algorithmic}[1]
\STATE {\bfseries Input:} Prefix end $p$, window end $w$, coarse stride $s_c$, fine stride $s_f$
\STATE \textit{\textbf{1. Coarse search sparsely scans the window to locate a promising boundary region}}
\STATE Construct coarse candidate set $\mathcal B_c \leftarrow \{p+1+s_c, p+1+2s_c, \dots, w\}$
\FOR{each $b \in \mathcal B_c$}
    \STATE Compute self-containedness $S(p,b) \leftarrow$ Eq.~\eqref{eq:self_contain_score}
\ENDFOR
\STATE Select coarse optimum $\tilde{b}$ from $\mathcal B_c$ by Eq.~\eqref{eq:boundary_selection_rule}
\STATE
\STATE \textit{\textbf{2. Fine search refines the boundary choice around the best coarse candidate}}
\STATE Define fine neighborhood $\mathcal B_f \leftarrow [\tilde{b}-s_c, \tilde{b}+s_c] \cap (p, w]$ with stride $s_f$
\FOR{each $b \in \mathcal B_f$}
    \STATE Compute self-containedness $S(p,b) \leftarrow$ Eq.~\eqref{eq:self_contain_score}
\ENDFOR
\STATE Select final boundary $b^\star$ by Eq.~\eqref{eq:boundary_selection_rule}
\STATE \textbf{return} $b^\star$
\end{algorithmic}
\end{algorithm}

\newpage
\section{Case Study 1: VSB vs. Fixed-size block Baseline} \label{appendix:case_study1}
We provide a case study in Figure~\ref{fig:case_study_raw} comparing the generation output and quality between our VSB and Fixed-size block decoding. Below is a detailed step by step analysis of the outputs with respect to a given question in GSM8K test set:

\textit{Question: According to its nutritional info, a bag of chips has 250 calories per serving. If a 300g bag has 5 servings, how many grams can you eat if your daily calorie target is 2000 and you have already consumed 1800 calories?}

\paragraph{Problem structure and sensitivity to segmentation.}
The question involves a short chain of dependent steps. First, the total calories in the bag are computed from the number of servings and calories per serving. Next, the remaining daily calorie budget is obtained from the target and the calories already consumed. Finally, this remaining budget is converted into a fraction of the bag and then into grams using the bag weight. The full reasoning chain can be seen as:
$$
5 \times 250 = 1250,\qquad
2000 - 1800 = 200,\qquad
\frac{200}{1250} \times 300 = 48,
$$
where each intermediate result is needed by the next step. If numbers, units, or equations are broken across blocks, the error can easily propagate to the final answer. This makes the example highly sensitive to how decoding blocks are segmented.
\paragraph{VSB decoding behavior.}
With VSB decoding, block boundaries are chosen based on semantic and logical completeness. In this case, each block ends after a complete clause or a finished calculation. For instance, illustrated in top panel of Figure~\ref{fig:case_study_raw}, the remaining calorie computation,
$
2000 - 1800 = 200,
$
appears entirely within a single block (Block 6). Likewise, the fraction:
$
\frac{200}{1250}
$
and its simplified form:
$
\frac{4}{25}
$
are kept intact before converting to grams. Because equations and explanations are not split mid-way, each block remains easy to interpret on its own. This allows the final conversion 
$
\frac{4}{25} \times 300 = 48
$
to be carried out correctly.

\paragraph{Numerical and unit consistency under VSB.}
VSB also preserves correct associations between numbers and their meanings. Calories per serving, total bag calories, remaining calories, and bag weight remain clearly separated. In particular, the $300$ grams value is correctly treated as the weight of the whole bag, not a serving. This behavior is consistent with the low average self-containedness divergence observed for VSB,
$$
\mathrm{Avg}\, S(p,b^\star) = 0.055,
$$
indicating that blocks can be understood without relying on future context.

\paragraph{Fixed-size decoding and fragmentation effects.}
In contrast, fixed-size decoding with a $64$-token budget places boundaries without considering semantic structure. Shown in the bottom panel of Figure~\ref{fig:case_study_raw}, blocks break in the middle of equations and sentences. One boundary interrupts the computation (Block 2-3):
$$
5 \times 250 = 1250,
$$
while another splits an explanatory sentence into incomplete parts  (Block 4-5). These fragmented blocks are harder to interpret and depend on later tokens to recover meaning. This is reflected in a higher self-containedness divergence,
$$
\mathrm{Avg}\, S(p,b) = 0.25.
$$
The fragmentation leads to a \textbf{clear error} later in the reasoning. The model incorrectly treats $300$ grams as the weight of a serving rather than the whole bag. This unit confusion directly affects the grams conversion and produces an incorrect final answer.

\begin{figure} [h!]
    \centering
    \includegraphics[width=0.8\linewidth]{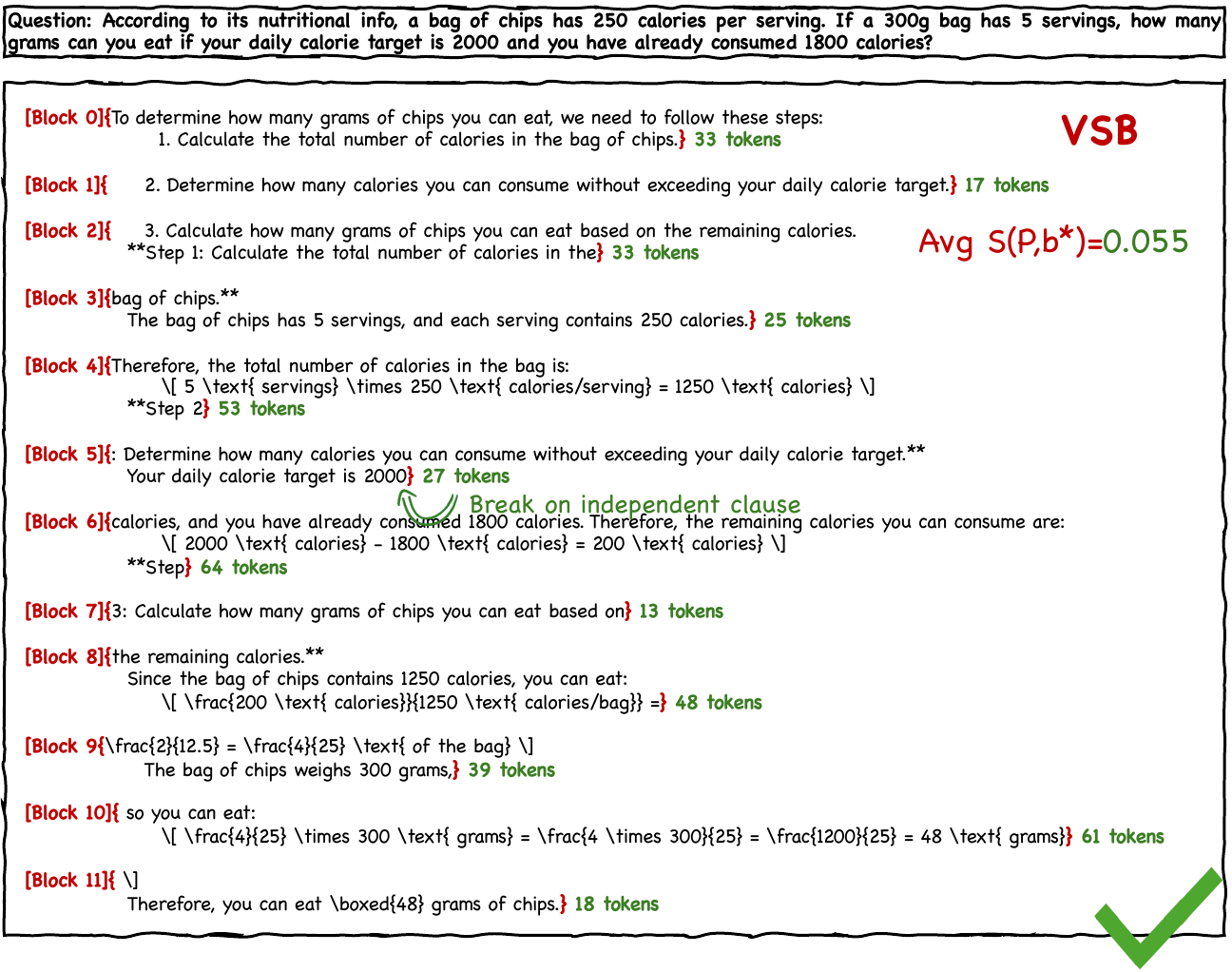}
    \includegraphics[width=0.8\linewidth]{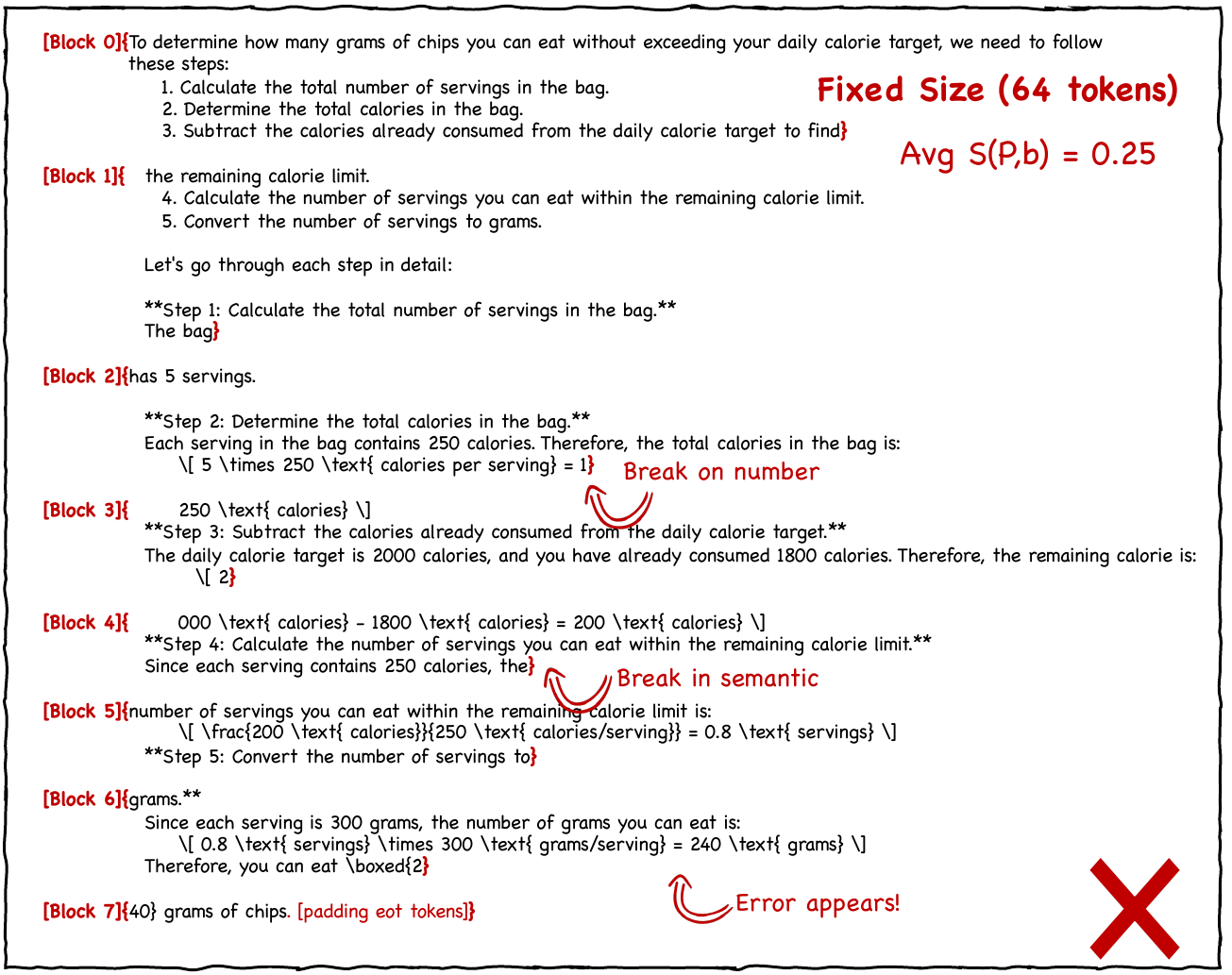}
    \caption{Case study of VSB vs. Fixed-size block baseline (64-tokens). VSB produces lower self-containedness divergence with more semantically coherent blocks at various sizes. Fixed-size blocks results in larger self-containedness divergence and blocks breaks on key number, semantics with errors in the final result.}
    \label{fig:case_study_raw}
\end{figure}

\section{Case Study 2: Self-containedness Guided Block Selection} \label{appendix:case_study2}
Figures~\ref{fig:casestudy2_overview},~\ref{fig:casestudy2_fullset_1}, and ~\ref{fig:casestudy2_fullset_2} provide a full decoded-token trace showing how VSB repeatedly applies the self-contained boundary rule across an entire solution. At each prefix position $p$, VSB evaluates candidate cutoffs $b \in (p, p+W]$ (here with a fixed window budget $W=64$), scores each candidate by self-containedness, and then selects the boundary $b^\star$ using a length-aware trade-off rather than the raw minimum of the divergence. Figures~\ref{fig:casestudy2_overview} provides an overview of case study 2's question and VSB's response annotated in blocks, as well as the overall self-containedness divergence across all blocks.

\paragraph{Overall pattern: boundaries align with semantic completion.}
Shown Figures~\ref{fig:casestudy2_fullset_1} and~\ref{fig:casestudy2_fullset_2} , across Blocks 0-20, the selected cutoffs repeatedly coincide with visually complete units in the decoded tokens. When the text forms a complete thought (e.g., finishing a sentence, closing an equation, completing an assignment like ``$a{=}1$''), the divergence for candidates ending at that point is low, and $S_{\text{L-A}}(p,b)$ peaks there. When the text is mid-construction (e.g., an unfinished math environment or an incomplete formula), $S(p,b)$ rises for longer candidates, and the trade-off selects an earlier cutoff that avoids committing future-dependent content.

A key qualitative observation across the blocks in Figures~\ref{fig:casestudy2_fullset_1} and ~\ref{fig:casestudy2_fullset_2} is that spikes in $S(p,b)$ coincide with tokens that clearly require continuation (unfinished equations, half-open LaTeX structures, or phrases that introduce content not yet stated). VSB responds by placing the boundary just before those unstable regions, so the next block can incorporate the missing future context before committing.

\paragraph{Where VSB looks strongest in this example.}
\begin{enumerate}
    \item \textbf{It separates stable prose from unstable math.} Early blocks show VSB committing short-to-medium spans that finish a statement, then opening the next block for the upcoming construction (Blocks 0-2). This matches the intended behavior: prose clauses are often self-contained, while upcoming math (definitions, substitutions, quadratic-formula expansions) creates future dependence until the expression is closed.
    \item \textbf{It avoids committing incomplete equations.} Throughout Blocks 3-14, the decoded content alternates between explanatory text and symbolic expansions. Candidate boundaries that would cut inside a formula tend to have higher divergence, while candidates that end at a natural closure (end of a displayed equation, after the full numerator/denominator is formed, after ``Thus,'' when the next line provides the result) yield lower divergence and therefore higher $S_{\text{L-A}}(p,b)$. This is precisely the behavior described in the main text: tokens beyond $b^\star$ are unstable because their meaning depends on the continuation from the next block.
    \item \textbf{It remains consistent across the whole generation.} In Blocks 15-20, the same pattern continues: VSB commits a self-contained definition of ``roots of unity", then postpones the claim about ``6th roots'' until the supporting context is in place. The final block includes the concluding sentence and answer, which is naturally self-contained and corresponds to low divergence near the end.
\end{enumerate}

\paragraph{Failure mode analysis (constructively).}
The main weakness appeared with VSB in the trace is that \textbf{some blocks are still very short} (e.g., 3-9 tokens in several places). Importantly, these short blocks are not arbitrary: they occur at positions where the decoded tokens themselves are genuinely small self-contained units, such as
\begin{itemize}[leftmargin=*,labelsep=0.0em,itemsep=0pt,parsep=0pt,topsep=0pt]
    \item local parameter assignments or short enumerations (e.g., ``$b{=}1$'', ``$c{=}1$'');
    \item LaTeX control-sequence glue where committing a longer span would cross into a future-dependent region (opening a fraction/sqrt that will only be completed by later tokens);
    \item short connective phrases that immediately precede a construction with high divergence (where longer candidates incur a large $(b-p)S(p,b)$ penalty).
\end{itemize}
So, even when a block is short, it is typically \emph{still self-contained}: committing it does not freeze an interpretation that depends on unseen continuation.

That said, \textbf{short blocks do have a practical cost}: they reduce parallelism and therefore limit speedups in exactly the regions where the model is most ``structure-sensitive'' (dense math or heavy formatting). In this case study, the shortest blocks appear clustered around the most syntactically constrained parts of the solution (quadratic formula expansion, nested radicals/fractions, and equation delimiters). This suggests a principled explanation: when the generation enters a highly future-dependent micro-region, the length-aware objective prefers to commit only the stable prefix of that region, rather than forcing a longer but unstable commitment. This issue can be potentially resolved by defining a ``minimum commit" value, nonetheless, we leave this as future practical tips instead of a fundamental analysis on the benefit of self-containedness.

\begin{figure}[t]
    \centering
    \begin{subfigure}[t]{0.8\linewidth}
        \centering
        \includegraphics[width=\linewidth]{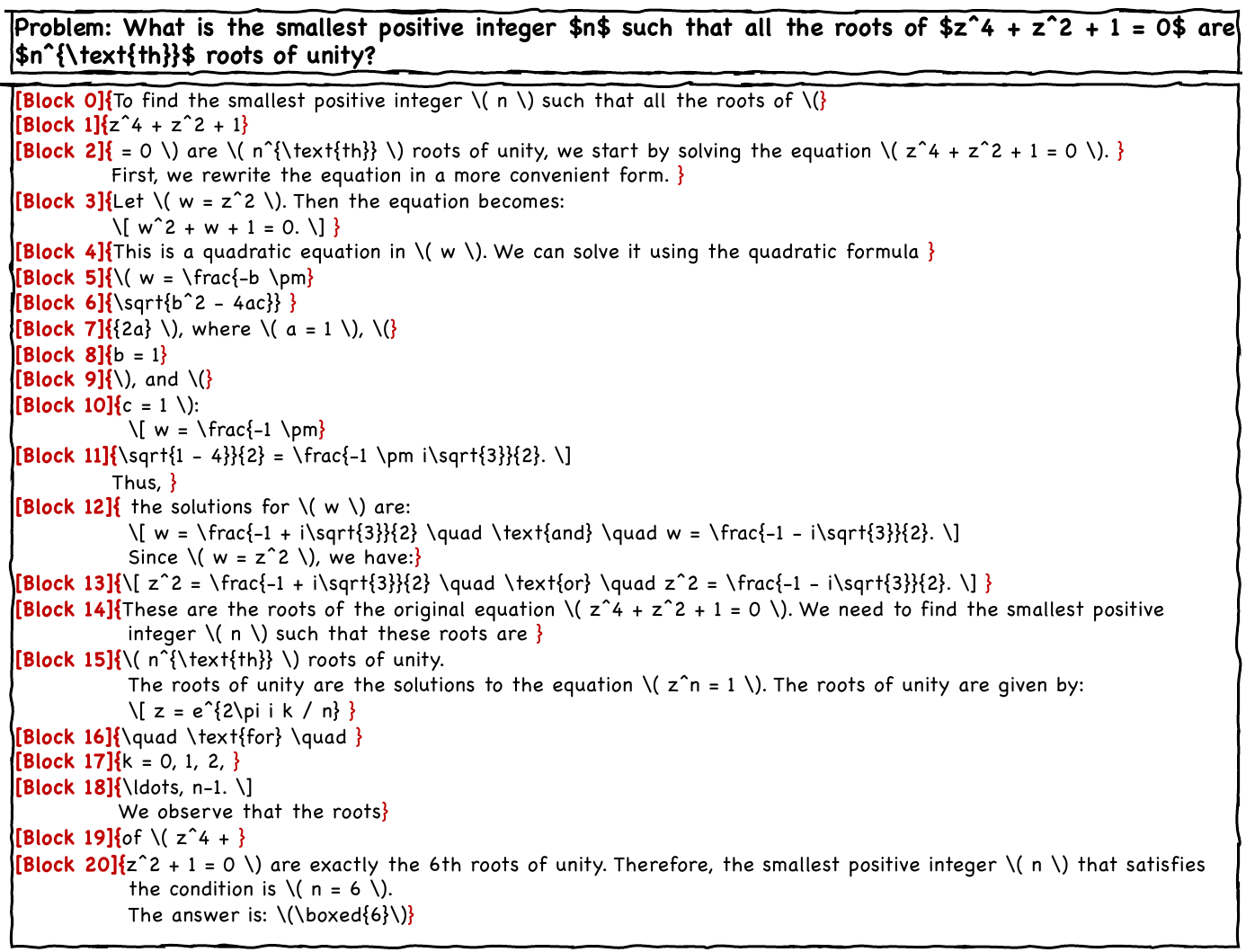}
        \caption{VSB response - all blocks}
    \end{subfigure}
    \hfill
    \begin{subfigure}[t]{0.8\linewidth}
        \centering
        \includegraphics[width=\linewidth]{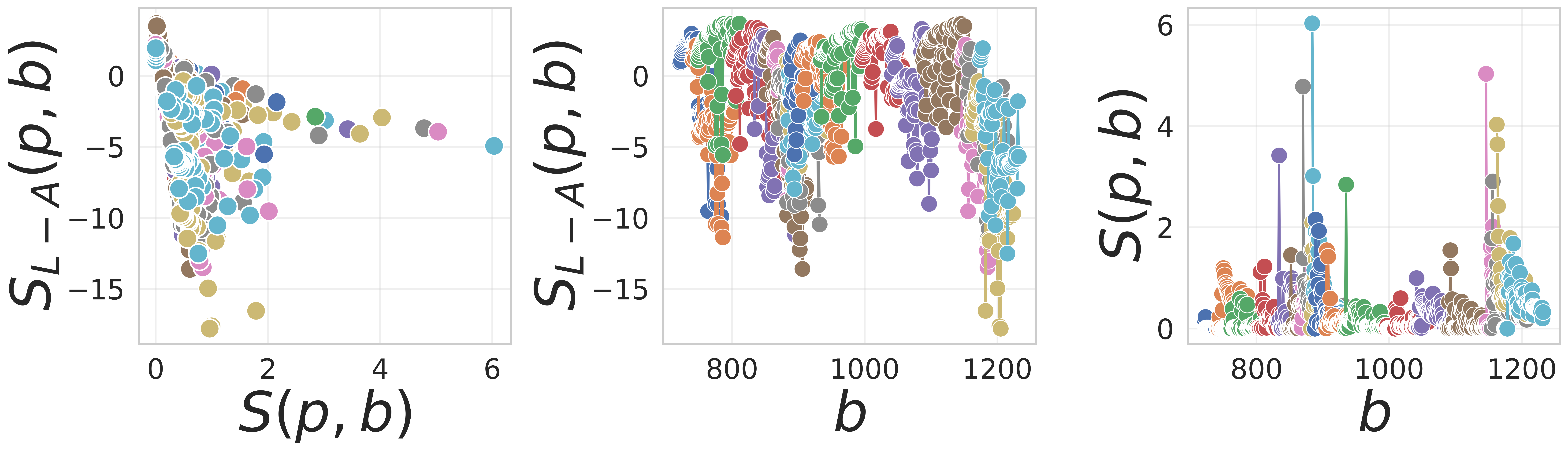}
        \caption{Self-containedness across all blocks}
    \end{subfigure}
    \caption{Case study 2 Overview}
    \label{fig:casestudy2_overview}
\end{figure}

\begin{figure}[t]
    \centering
    \begin{subfigure}[t]{0.6\linewidth}
        \centering
        \includegraphics[width=0.8\linewidth]{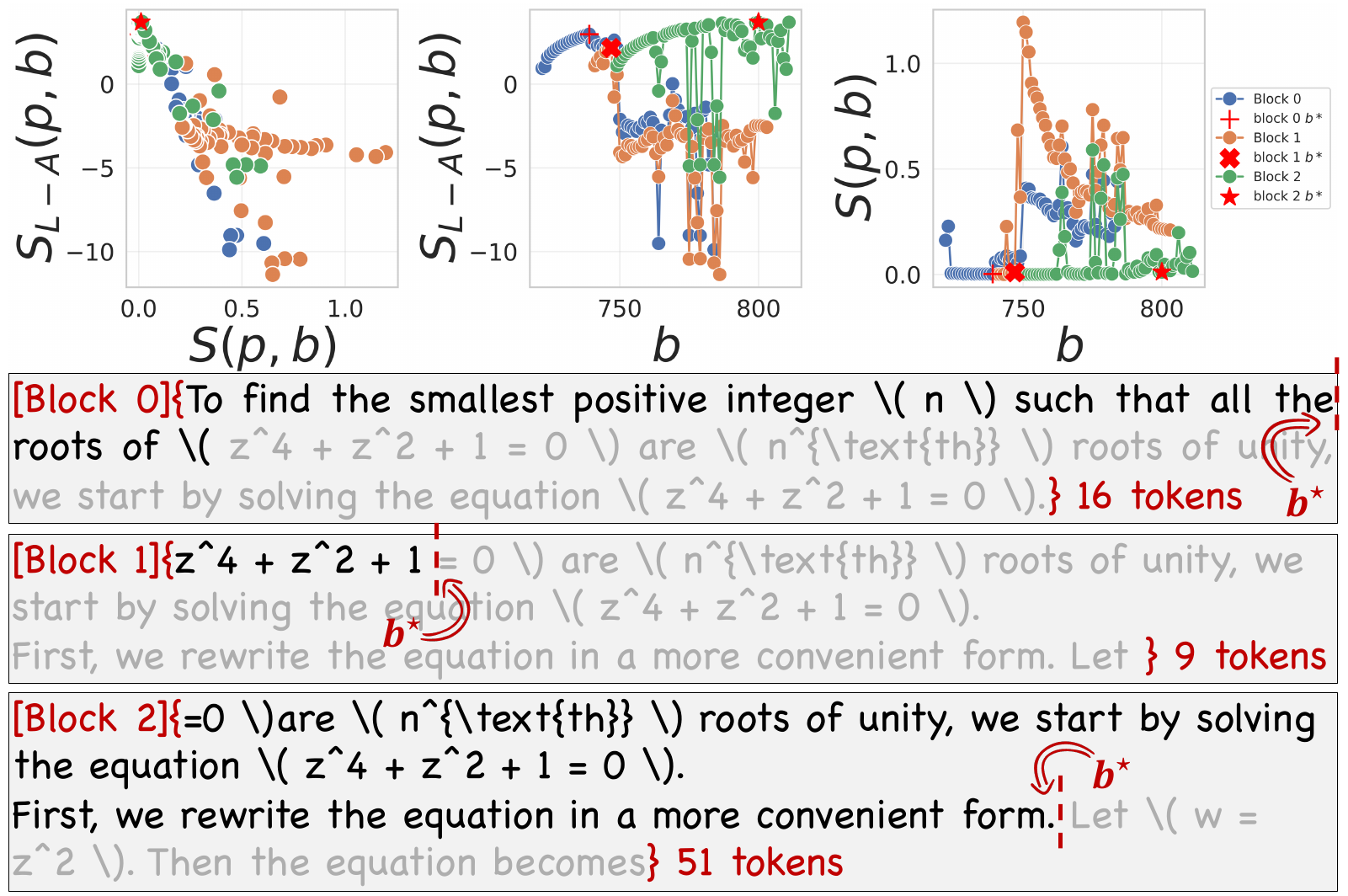}
        \caption{Block 0-2}
    \end{subfigure}

    \begin{subfigure}[t]{0.49\linewidth}
        \centering
        \includegraphics[width=0.9\linewidth]{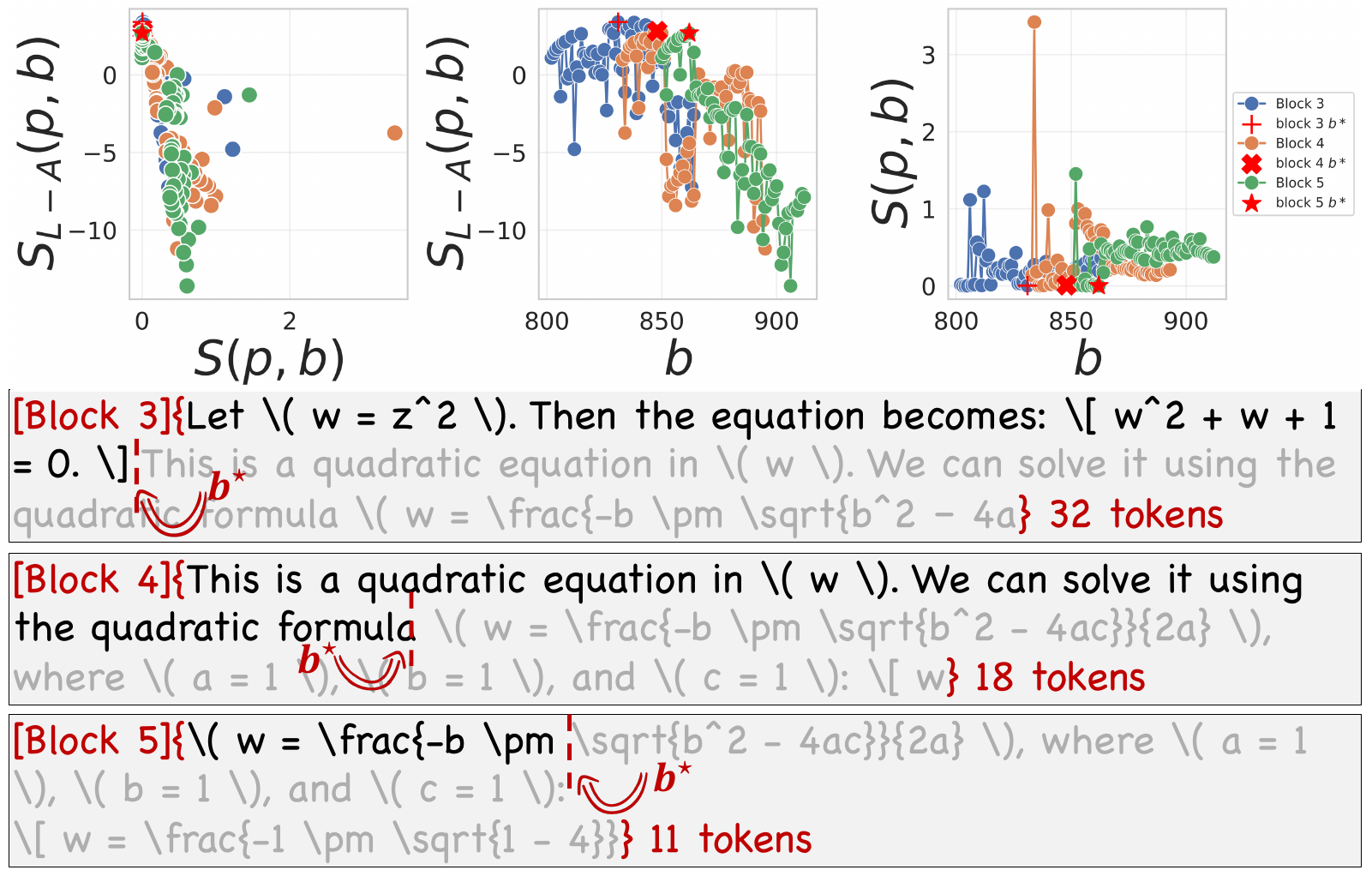}
        \caption{Block 3-5}
    \end{subfigure}
    \hfill
    \begin{subfigure}[t]{0.49\linewidth}
        \centering
        \includegraphics[width=0.9\linewidth]{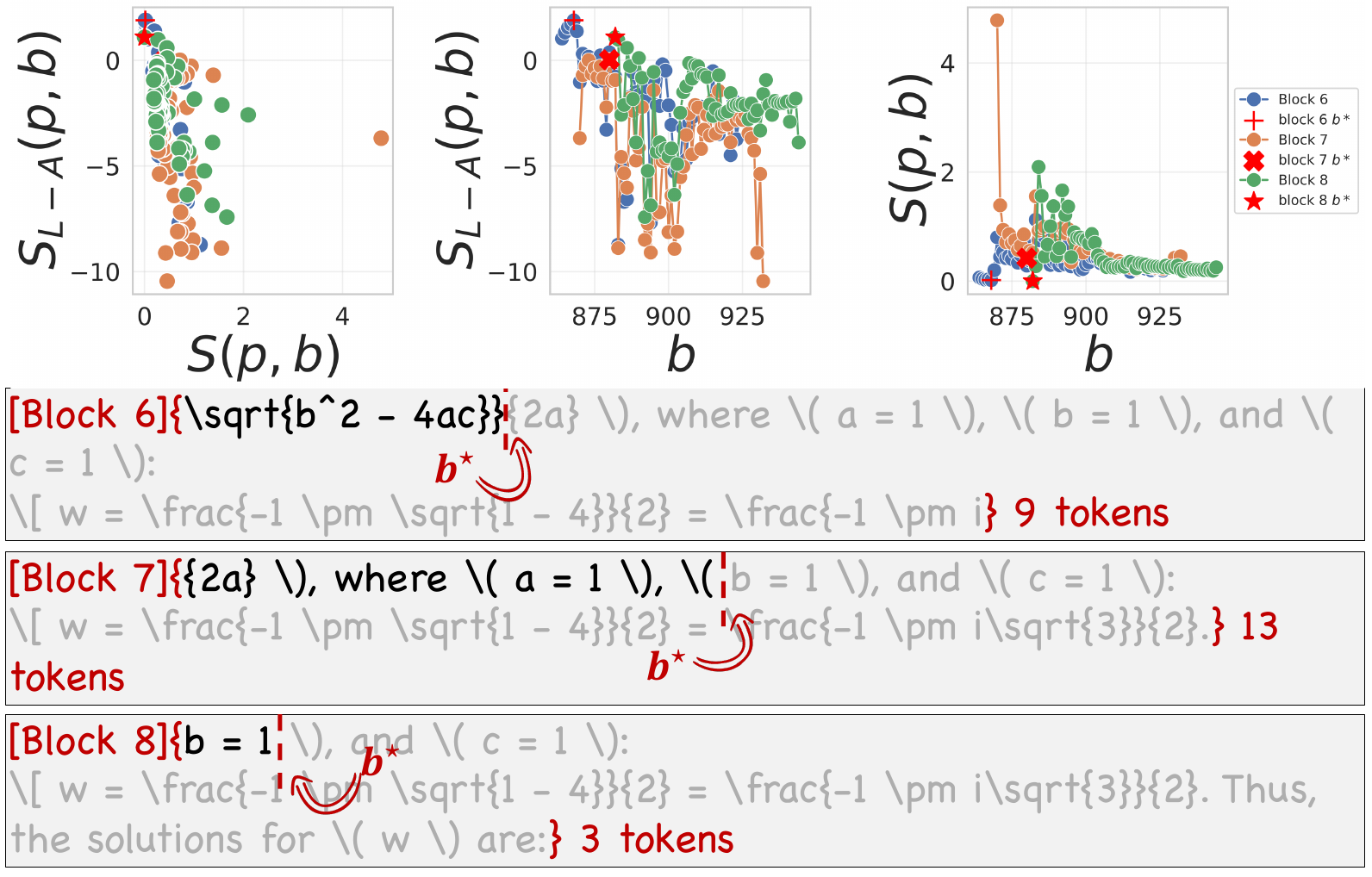}
        \caption{Block 6-8}
    \end{subfigure}

        \begin{subfigure}[t]{0.49\linewidth}
        \centering
        \includegraphics[width=0.9\linewidth]{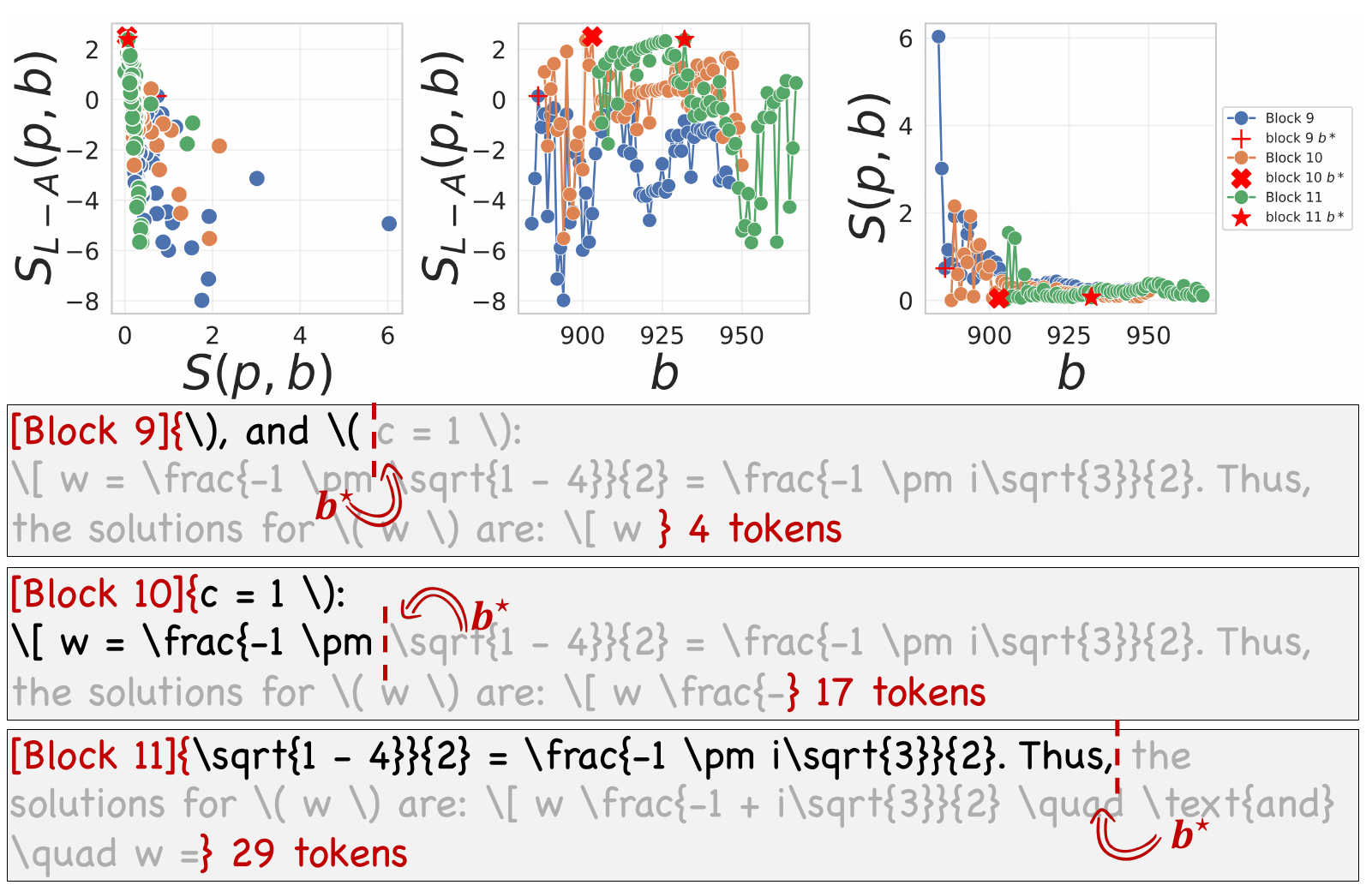}
        \caption{Block 9-11}
    \end{subfigure}
    \hfill
    \begin{subfigure}[t]{0.49\linewidth}
        \centering
        \includegraphics[width=0.9\linewidth]{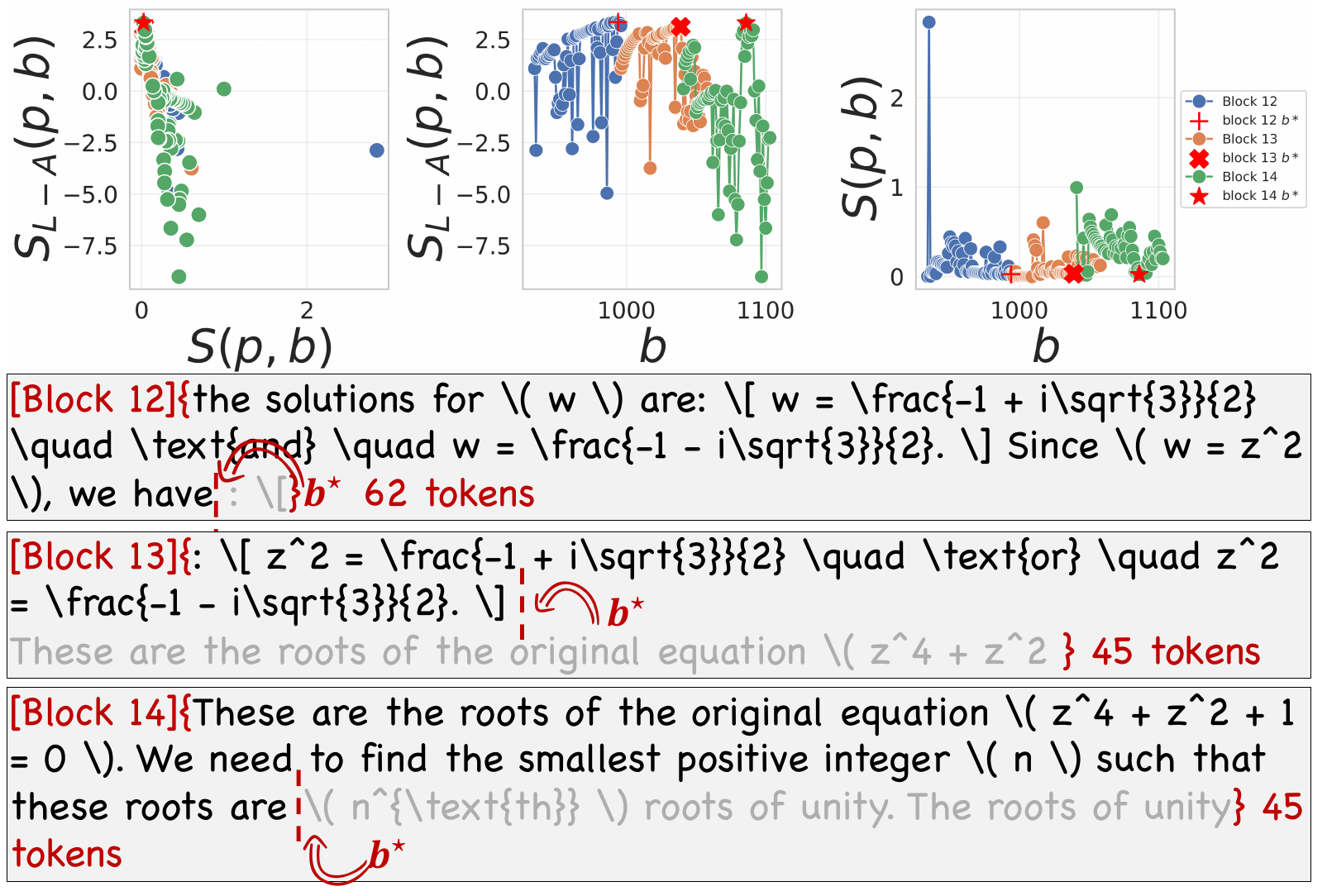}
        \caption{Block 12-14}
    \end{subfigure}

    \caption{Self-containedness divergence and corresponding decoded tokens with self-contained boundary highlighted.}
    \label{fig:casestudy2_fullset_1}
\end{figure}

\begin{figure}[t]
    \centering
    \begin{subfigure}[t]{0.49\linewidth}
        \centering
        \includegraphics[width=0.9\linewidth]{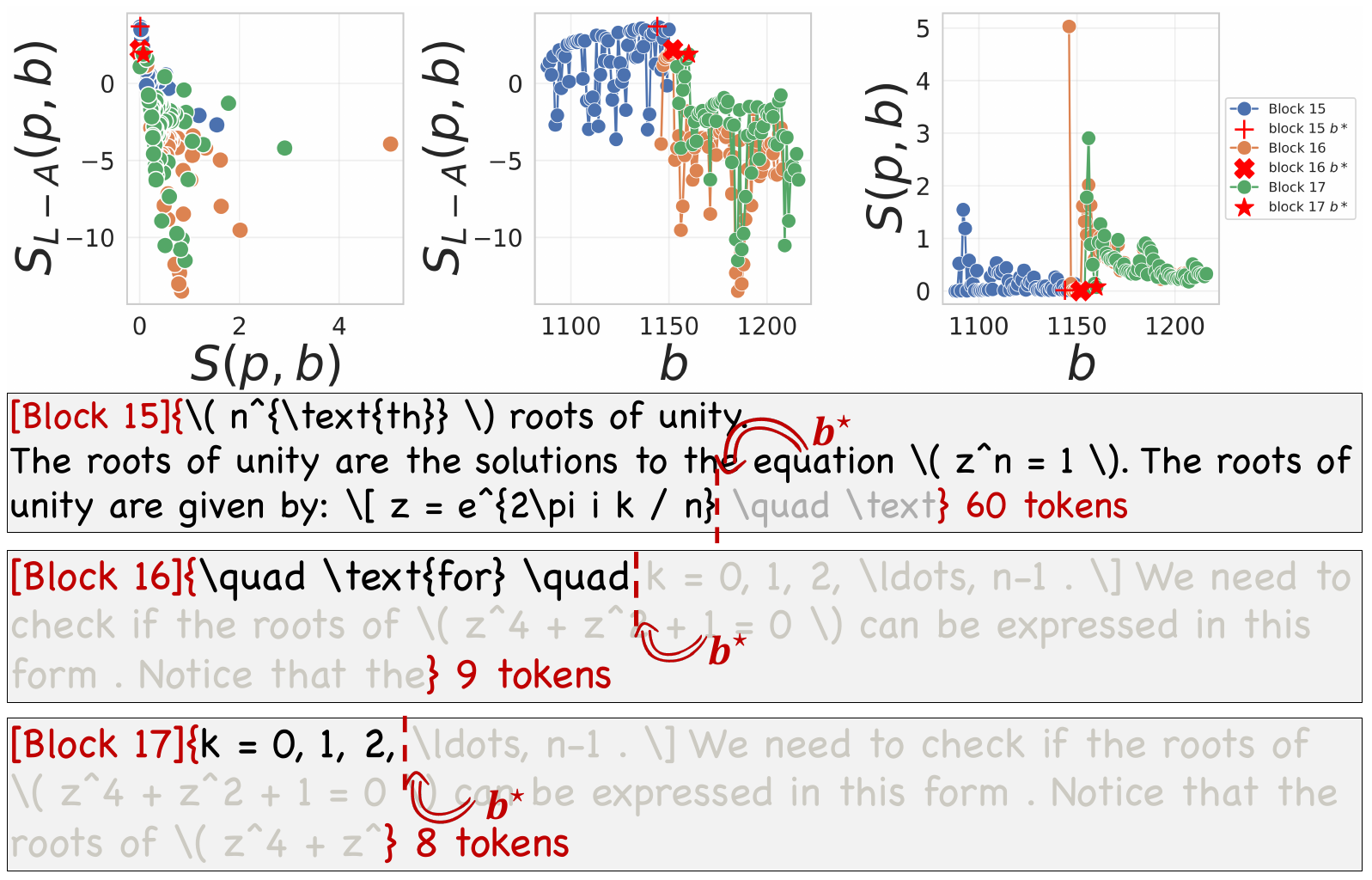}
        \caption{Block 15-17}
    \end{subfigure}
    \hfill
    \begin{subfigure}[t]{0.49\linewidth}
        \centering
        \includegraphics[width=0.9\linewidth]{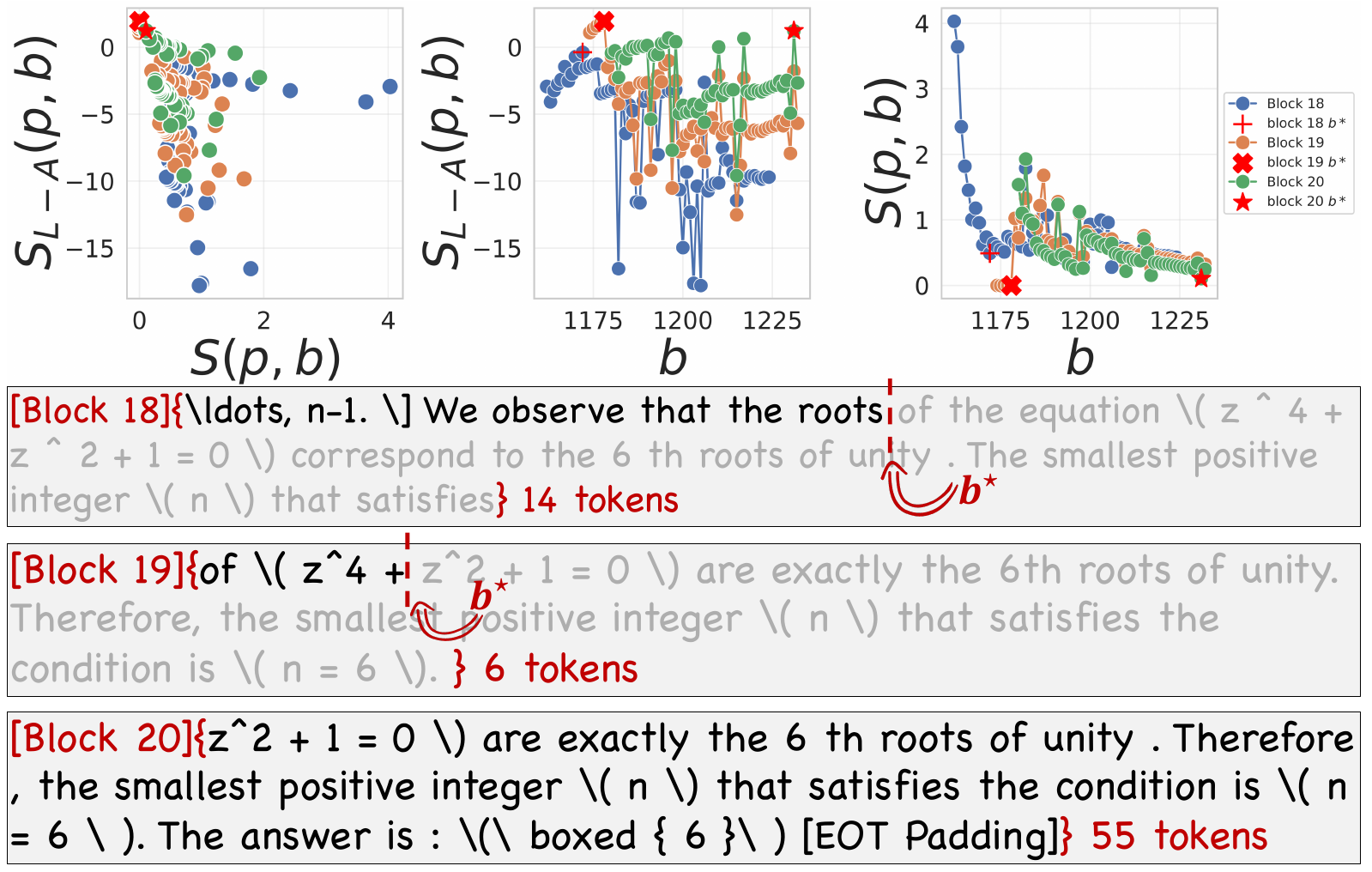}
        \caption{Block 18-20}
    \end{subfigure}
    \caption{Self-containedness divergence and corresponding decoded tokens with self-contained boundary highlighted - Continuation.}
    \label{fig:casestudy2_fullset_2}
\end{figure}
\end{document}